\documentclass[draftcls,onecolumn,journal]{IEEEtran}
\usepackage{amsmath,amsfonts}
\usepackage{array}
\usepackage{textcomp}
\usepackage{stfloats}
\usepackage{url}
\usepackage{verbatim}
\usepackage{graphicx}
\usepackage{cite}
\usepackage{soul}
\input{TIT.sty}

\newcommand{\lp}{\left(}
\newcommand{\rp}{\right)}
\newcommand{\lb}{\left[}
\newcommand{\rb}{\right]}
\newcommand{\lbp}{\left\{}
\newcommand{\rbp}{\right\}}
\newcommand{\lba}{\left\lvert}
\newcommand{\rba}{\right\rvert}

\newcommand{\mcal}{\mathcal}

\newcommand{\mbb}{\mathbb}
\newcommand{\mrm}{\mathrm}

\newcommand{\lce}{\left\lceil}
\newcommand{\rce}{\right\rceil}
\newcommand{\lfl}{\left\lfloor}
\newcommand{\rfl}{\right\rfloor}
\makeatletter
\newcommand{\vast}{\bBigg@{3}}
\newcommand{\Vast}{\bBigg@{4}}
\makeatother

\newcommand{\E}{\mathbb{E}}

\renewcommand{\Pr}{\mathbb{P}}

\newcommand{\Indc}[1]{\mathds{1}\lbp #1\rbp}

\newcommand{\argmin}{\mathop{\mathrm{argmin}}}
\newcommand{\argmax}{\mathop{\mathrm{argmax}}}

\makeatletter
\newcommand*{\indep}{%
  \mathbin{%
    \mathpalette{\@indep}{}%
  }%
}
\newcommand*{\nindep}{%
  \mathbin{
    \mathpalette{\@indep}{\not}
  }%
}
\newcommand*{\@indep}[2]{%
  \sbox0{$#1\perp\m@th$}
  \sbox2{$#1=$}
  \sbox4{$#1\vcenter{}$}
  \rlap{\copy0}
  \dimen@=\dimexpr\ht2-\ht4-.2pt\relax
  \kern\dimen@
  {#2}%
  \kern\dimen@
  \copy0 
} 
\makeatother

\newtheorem{theorem}{Theorem}
\newtheorem{corollary}{Corollary}

\newtheorem{remark}{Remark}

\newtheorem{property}{Property}

\newlist{philist}{itemize}{1}
\setlist[philist]{label=\textbf{$\Phi_{1}$: }}

\usepackage{multirow}

\begin{document}

\title{Detection Augmented Bandit Procedures for Piecewise Stationary MABs: A Modular Approach}

\author{Yu-Han Huang, Argyrios Gerogiannis,~\IEEEmembership{Graduate Student Member,~IEEE}, Subhonmesh Bose,~\IEEEmembership{Member,~IEEE}, Venugopal V. Veeravalli,~\IEEEmembership{Fellow,~IEEE}
\thanks{This work  was supported in part by a grant from the C3.ai Digital Transformation Institute, and in part by the Army Research Laboratory under Cooperative Agreement W911NF-17-2-0196, through the University of Illinois at Urbana-Champaign.}
\thanks{The authors are with the ECE department and the Coordinated Science Laboratory, Grainger College of Engineering, University of Illinois at Urbana-Champaign. Email:\{yuhanhh2,ag91,boses,vvv\}@illinois.edu.}

}

\IEEEpubid{0000--0000~\copyright~2023 IEEE}

\maketitle

\begin{abstract}
Conventional Multi-Armed Bandit (MAB) algorithms are designed for stationary environments, where the reward distributions associated with the arms do not change with time. In many applications, however, the environment is more accurately modeled as being non-stationary. In this work, piecewise stationary MAB (PS-MAB) environments are investigated, in which the reward distributions associated with a subset of the arms change at some change-points and remain stationary between change-points.  Our focus is on the asymptotic analysis of PS-MABs, for which practical algorithms based on change detection have been previously proposed. Our goal is to modularize the design and analysis of such Detection Augmented Bandit (DAB) procedures. To this end, we first provide novel, improved performance lower bounds for PS-MABs. Then, we identify the requirements for stationary bandit algorithms and change detectors in a DAB procedure that are needed for the modularization. 
We assume that the rewards are sub-Gaussian. 
Under this assumption and a condition on the separation of the change-points, we show that the analysis of DAB procedures can indeed be modularized, so that the regret bounds can be obtained in a unified manner for various combinations of change detectors and bandit algorithms. Through this analysis, we develop new modular DAB procedures that are order-optimal.  Finally, we showcase the practical effectiveness of our modular DAB approach in our experiments, studying its regret performance compared to other methods and investigating its detection capabilities.
\end{abstract}


\section{Introduction}
\label{sec:intro}
\IEEEPARstart{I}{n} the Multi-Armed Bandit (MAB) problem, an agent chooses an arm $a_t$ from a finite set $\lb A \rb \subset \mbb{N}$ at each time $t \in \mbb{N}$ and obtains a stochastic  reward $X_{a_t,t}$, i.i.d. across time. We use  $\mbb{N}$ to denote the set of natural numbers and $\lb n \rb:=\lbp 1, \dots, n \rbp$ for $n \in \mbb{N}$. The agent employs a policy adapted to the history of actions and observations up to that point,
with an aim to maximize  (expected) accumulated reward over a horizon of length $T$. 
See \cite{lattimore2020bandit,slivkins2019introduction} for recent books on the topic. A variety of engineering decision-making problems can be mapped to a bandit problem instance, e.g., in recommendation systems \cite{li2010recomm,lefortier2014recomm}, online advertising \cite{chapelle2011advr,sertan2012advr,schwartz2017onlineadvr}, dynamic pricing \cite{tajik2024dynamicpr} and real-time bidding \cite{flajolet2017bidding}.
In the most common setting of the MAB problem, the reward distributions of $X_{a,t}$'s associated with all arms $a \in [A]$ remain unchanged over $T$. Such stationarity rarely holds in practice. For example,  in the context of recommendation systems, preferences of users can change over time due to changing fashions and trends. This observation has spurred interest in the analysis of nonstationary MABs; see \cite{cai2017nsbid,lu2019nsbid2,chen2020trafficns} for examples.

As an initial step towards addressing non-stationary MABs, multiple prior works have focused on piecewise stationary MABs (PS-MABs) \cite{kocsis2006discounted} in which the reward distributions associated with a subset of the arms change across $N_T$ change-points at $\nu_{0}\coloneqq 1 < \nu_1 < \ldots < \nu_{N_T} < \nu_{N_{T} + 1} = T + 1$. Over the $k^{\mathrm{th}}$ interval $\lbp \nu_{k-1}, \dots, \nu_{k} - 1 \rbp$, the reward distributions for all arms remains the same where $\lp X_{a,t}\rp_{t=\nu_{k-1}}^{\nu_{k}-1}$ are assumed to be $\sigma^2$-sub-Gaussian\footnote{A random variable $X$ is said to be $\sigma^{2}$-sub-Gaussian if its cumulant generating function $\phi_{X}$ is upper bounded by that of a Gaussian random variable with mean $0$ and variance $\sigma^{2}$, i.e., for any $\theta\in\mbb{R}$, $
\phi_{X}\lp\theta\rp\leq\frac{\sigma^{2}\theta^{2}}{2}$.} with mean $\mu_{a, k}$. Across the change-point $\nu_{k}$, at least one of the arms experiences a mean-shift, i.e., $\max_{a \in [A]} \lba \mu_{a, k + 1} - \mu_{a, k} \rba >0$. Then, an agent in an PS-MAB environment seeks a causal control policy that minimizes the (dynamic) regret over horizon $T$, defined as 
\begin{equation}
R_{T}\coloneqq\E\lb\sum_{k=1}^{N_{T} + 1}\sum_{t=\nu_{k-1}}^{\nu_{k}-1} \left( \max_{a \in [A]} \mu_{a, k} - \mu_{a_t, k} \right) \rb.\label{eq:regret}
\end{equation}  
The PS model has been argued to be a good approximation for many real-world environments in  \cite{auer2002using,seznec2020single}.

An algorithm designed for a stationary MAB environment (henceforth referred to as \emph{stationary algorithm}), such as the Upper Confidence Bound (UCB) algorithm has little chance to succeed against an PS-MAB instance. To see why, notice that a stationary algorithm must quickly identify an `optimal arm' within an interval and refrain from pulling sub-optimal arms too often in order to minimize regret over that interval. By pulling these pre-change sub-optimal arms infrequently, the algorithm will fail to quickly identify mean-shifts in their rewards. This failure will adversely affect regret when such arms become optimal following a change-point.
In a nutshell, blindly running a stationary algorithm results in poor performance in an PS-MAB environment; one must suitably respond to changes. In this paper, we study an algorithmic framework for control design in PS-MAB environments that combines the extensive literature on stationary bandit algorithms (e.g., see \cite{auer2002finite,audibert2009minimax,lattimore2020bandit}) and (quickest) change detection theory (e.g., see \cite{tartakovsky-veeravalli-2005,Poor_Hadjiliadis_2008,vvv_qcd_overview}) in a modular framework.

Observations from PS-MAB environments prior to a change-point do not necessarily follow the reward distributions after it. Hence, control policies for PS-MABs must devise mechanisms to somehow forget such observations. There are two main approaches to forgetting. The first among these weigh the information from past observations `less' than those from recent arm pulls, forcing the algorithm to continuously adapt to the changing environment. Examples include the Discounted-UCB (D-UCB) algorithm 
in \cite{kocsis2006discounted,garivier2011upper}
and the Sliding-Window UCB (SW-UCB) algorithm in \cite{garivier2011upper}.
These methods require one to tune hyper-parameters such as a discount factor in D-UCB and a window-size in SW-UCB,  based on prior knowledge of the number of change-points that can occur over a finite horizon. Such information may not always be available in practice.

The second approach involves resetting the observations of rewards from some or all of the arms at specific time-points. One can predefine a schedule for such restarts, if the rate of changes is roughly known, e.g., as in \cite{besbes_mab2014}. Alternately, one can restart (stationary) learning, whenever a change is detected. Such algorithms essentially combine elements of a stationary bandit algorithm with that of a change detector used on an observation sequence or a function thereof. These detection-augmented bandit (DAB) algorithms can attempt to detect when the optimal arm might have changed to trigger a reset, e.g., as in \cite{abbasi2023new}. Alternately, they can restart whenever a change is detected in the reward distribution from any of the arms, e.g., as in \cite{liu2018change,cao2019nearly,besson2022efficient}. Upon detecting a change in the reward distribution from a subset of the arms, one can choose to forget past observations from only those arms as in \cite{besson2022efficient} or trigger a global reset of observations from all arms, e.g., as in \cite{liu2018change,cao2019nearly,besson2022efficient}.

\begin{figure}
    \centering
    \includegraphics[width=0.5\linewidth]{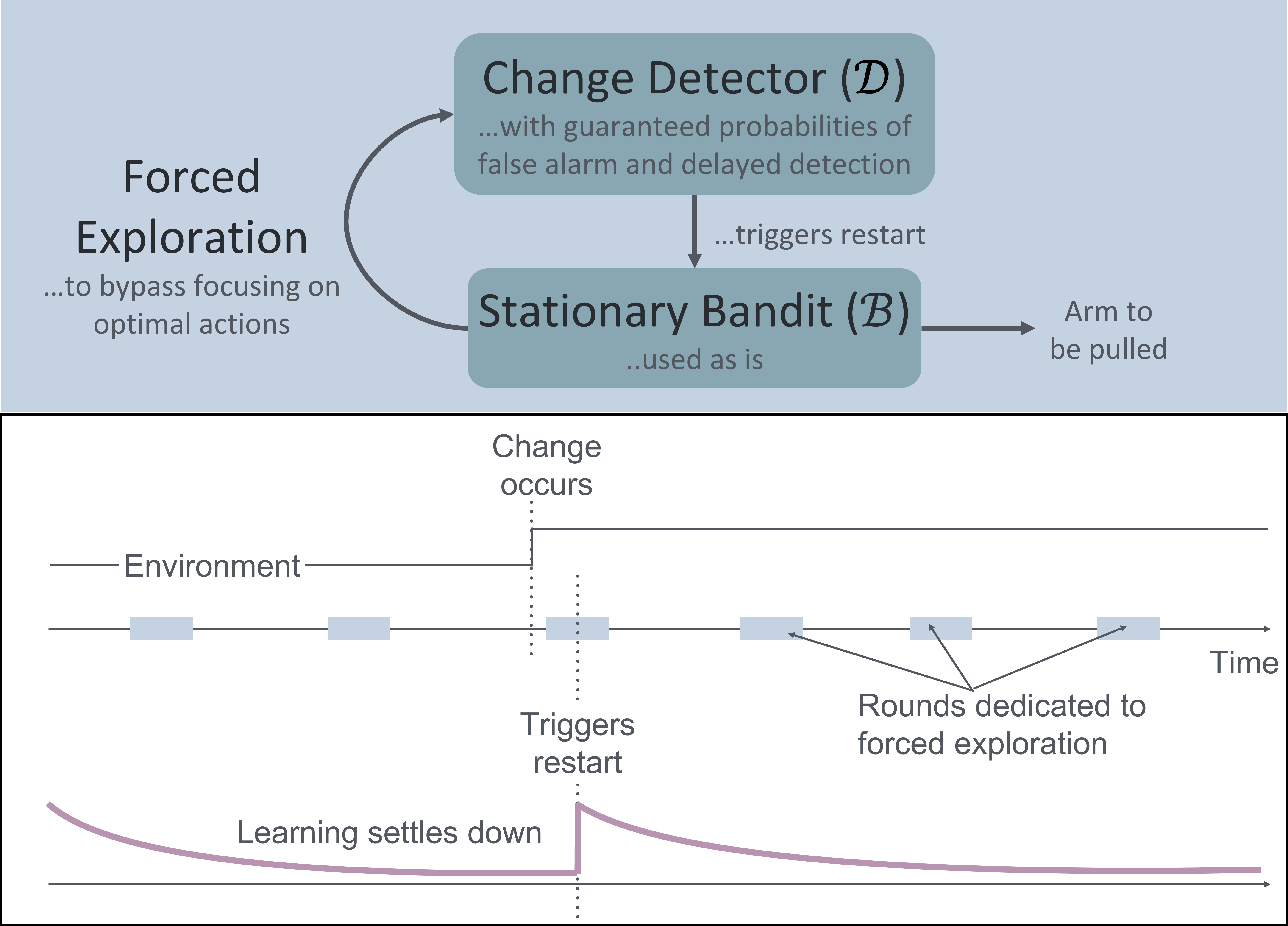}
    \caption{The general DAB procedure.}
    \label{fig:DAB}
\end{figure}
In this paper, we study how one can combine \emph{any} stationary bandit algorithm $\mcal{B}$ and \emph{any} change detector algorithm $\mcal{D}$, provided they satisfy certain properties, to obtain a DAB algorithm for PS-MABs. In addition to providing a modular algorithm design, we also provide a modular regret analysis that utilizes the regret growth of $\mcal{B}$ to produce a regret rate for our DAB procedure. More precisely, order-optimal $\mcal{O}\lp \sqrt{AT\log(T)} \rp$ (instance-independent) regret growth for $\mcal{B}$ yields $\mcal{O}\lp \sqrt{A N_T T\log(T)} \rp$ regret growth for DAB procedures\footnote{We use the notation $f \lesssim g$ to denote $f = \mcal{O}\lp g \rp$, $f \gtrsim g$ to denote $f = \Omega\lp g \rp$, and $f \simeq g$ to denote $f = \Theta\lp g \rp$ for arbitrary functions $f$ and $g$, where we borrow the standard notions for $\mcal{O}$ $\Theta$, and $\Omega$.}. For reasonable choices of $\mcal{B}$, such as the UCB and klUCB algorithms, suboptimal arms are pulled infrequently at a rate of $\mcal{O}\lp\log\lp T\rp\rp$ over a horizon of length $T$, per \cite{lattimore2020bandit}. In other words, $\mcal{B}$ learns an optimal arm quickly and neglects sub-optimal arms over time. However, this sublinearity limits the number of reward samples observed from suboptimal arms, making it harder for a change detector $\mcal{D}$ to identify possible changes in their reward distributions. Thus, our DAB procedure involves  \textit{forced exploration} of \emph{all} arms to effectively respond to changes as illustrated in Figure \ref{fig:DAB}. The rate of such an exploration is carefully controlled to limit how it affects the regret over a stationary interval, but also to reduce delays in detection after a change point or the possibilities of false alarms.

We now discuss prior algorithms for PS-MABs to better contextualize our work. Adapt-EVE in \cite{hartland2006multi} combines UCB with a Page-Hinkley change-detection test \cite{page1954continuous}, CuSum-UCB in \cite{liu2018change} combines UCB with a two-sided CuSum test from \cite{lorden1971procedures}, and Monitored UCB (M-UCB)  in \cite{cao2019nearly} combines UCB with a test that declares a change once the absolute difference between the empirical means of two windows surpasses some threshold. We remark that these algorithms require tunable parameters that require knowledge of the number of change-points or the magnitude of changes to achieve optimal performance. In contrast, the GLR-klUCB procedure proposed by \cite{besson2022efficient} and the the Adaptive Switching procedure (Adswitch) proposed by \cite{auer2019achieving} do not require such a priori knowledge of the non-stationarity, except GLR-klUCB requires conditions on the separation between change-points. However, AdSwitch's computational complexity is much higher than the former's and performs worse than stationary alternatives in experiments \cite{besson2022efficient}. Perhaps the common thread in all these prior works is that they combine a specific $\mcal{B}$ with a specific $\mcal{D}$. Our framework is \emph{modular} in that we allow any combination of $\mcal{B}$ and $\mcal{D}$, and provide a modular analysis of its regret under certain conditions. The art lies in utilizing controlled forced exploration in a way that allows us to disentangle the common history of reward observations that both $\mcal{B}$ and $\mcal{D}$ rely upon. We emphasize that while our work was partly motivated to leverage the long literature on change detection theory, we found that the performance metrics used in that literature are somewhat different from what $\mcal{D}$ requires to be effective in our DAB framework. In other words, identification of the requirements on $\mcal{D}$ (using the recent results in \cite{huang2025sequential}) stands as an important contribution of this work.  We show that the popular Generalized Likelihood Ratio (GLR) test and a Generalized Shiryaev-Roberts (GSR) test satisfy these requirements.

A prior modular paradigm is MASTER \cite{wei2021non}, which is designed as a black-box wrapper for a wide range of non-stationary reinforcement learning settings, including the PS-MAB case, and does not require prior knowledge of the non-stationarity. However, applying MASTER is not straightforward: its assumptions require additional analysis for each specific input algorithm $\mathcal{B}$, and in fact these assumptions have only been verified for UCB \cite{auer2002finite}. This requirement undermines modularity and prevents the straightforward substitution of algorithms. In contrast, our framework does not impose such restrictions and allows the use of any order-optimal $\mathcal{B}$ without additional proofs. Another key difference is that MASTER employs a fixed non-stationarity detection procedure, while our approach is compatible with a variety of change-detection procedures. Finally, perhaps the most significant drawback of MASTER is its practical behavior: as shown in \cite{gerogiannis2024ipfnsrlopt}, its detection test fails to trigger for realistic choices of the horizon (less than 143 billion). Empirical comparisons further demonstrate that DAB procedures consistently outperform MASTER \cite{gerogiannis2024ipfnsrlopt}, as additionally corroborated by our experiments in Section \ref{sec:exp}.

To make this article self-contained, we include universal lower bounds on regret for PS-MABs, which sharpen known lower bounds for PS-MABs in the instance-independent setting, as well as provide a novel instance-dependent lower bound. In doing so, we show the different roles played by instance-dependent and instance-independent bounds and pose an open question. Our analysis of DAB procedures is itself modular in that the regret upper bound cleanly separates into terms that arise from specific considerations---possible false alarms, (controlled) delays in detection, regret of the stationary algorithm $\mcal{B}$,  and forced exploration. In contrast to prior art \cite{liu2018change,cao2019nearly,auer2019adswitch,besson2022efficient}, we avoid `mixing' the analyses of $\mcal{B}$ and $\mcal{D}$ to produce regret bounds for DAB procedures.
Our procedures, as will become clear in the sequel, are designed to be prior-free and require minimal parameter tuning. The regret upper bound for DAB procedures in Theorem \ref{thm:det_reg} is order-optimal, as certified by our instance-independent lower bound for PS-MABs. This bound is derived under a condition on the separation between change points, similar to Assumption 4 in \cite{besson2022efficient}. We remark that such a separation makes \emph{all} changes detectable with high probability. Our simulations in Section \ref{sec:exp} reveal that such a condition is not required for the algorithm to perform well as it tends to detect subsequent changes, even when the assumptions are violated and the algorithm fails to detect a particular change-point. In addition to enjoying order-optimality and offering a modular analysis, we empirically demonstrate that the modular DAB procedures achieve regret performance commensurate with other methods proposed in the literature, depending on the choices of $\mathcal{B}$ and $\mathcal{D}$. We point out that our DAB procedure restarts $\mcal{B}$ when $\mcal{D}$ detects a change in the reward distribution in \emph{any} arm. We do not prescribe waiting to see if the optimal arm has changed to trigger a restart because  solving the question of best arm identification post-change is order-wise similar to that of learning to control in a stationary MAB environment. In addition, we discard the history of observations of all arms when a change in reward distributions is detected in any arm. One hopes to perform better through selective resets in said history of a subset of arms only if one has prior information on the nature of changes one can expect.\footnote{As shown in prior work \cite{besson2022efficient}, selectively resetting the histories of a subset of arms rather than globally, does not offer a clear practical advantage in truly prior-free settings and can sometimes be worse than global resetting.} In a prior-free setting, the context we design our DAB procedures for, our style of resets still yields order-optimal regret in theory and comparable or better performance than state-of-art in simulations.

\section{Lower Bounds on Regret Accumulation for PS-MABs}

\label{sec:prob}

Our goal in this paper is an order-optimal modular algorithmic framework for PS-MABs. We begin by presenting a lower bound on regret accumulation for any reasonable algorithm. Throughout this paper, we adopt the following convention for the mean rewards of various arms. Assume that the suboptimality gaps, $\Delta_{a, k}\coloneqq \max_{a'\in \lb A \rb}\mu_{a', k}- \mu_{a, k}$ and the change gaps $\Delta_{\mrm{c}, k}=\max_{a\in[A]}\lvert\mu_{a, k + 1}-\mu_{a, k}\rvert$ over the intervals $k \in \lb N_T\rb$ are uniformly bounded as
\begin{align}
    \Delta_{a,k} \leq C, \; \Delta_{\mrm{c},k}\geq\underline{\Delta}_{\mrm{c}}.
\end{align}
Thus, we limit the maximum extent of per-step regret from pulling any arm and the minimum degree of mean-shifts from at least one arm across a change-point. The first assumption is standard in the literature on stationary bandits; without it, the worst-case regret can be infinite. The second one ensures that the PS-MAB is nontrivial and the changes are large enough to both warrant and allow detection. We assume that the agent has no knowledge about $C$ and $\underline{\Delta}_{\mrm{c}}$, but they feature in our regret bounds.

%



With this notation, we now present two types of lower bounds on asymptotic regret growth. The first type explicitly depend on the suboptimality gaps of bandit instances and are referred to as \emph{instance-dependent} bounds. The second kind do not explicitly depend on said gaps and are known as \emph{minimax} or instance-independent bounds. We remark that Theorem 3 of \cite{garivier2011upper} provides a loose minimax regret lower bound of the order $\sqrt{T}$, which does not depend on the number of changes $N_{T}$. Proposition 4 of \cite{seznec2020single} provides a minimax  lower bound of the order $\sqrt{N_{T}AT}$ for a subclass of PS-MABs, in which the mean reward of each arm decreases with time. We present the following two results on instance-dependent and minimax lower bounds on regret for general PS-MABs. Without loss of generality and to simplify the notation, we assume that $T$ is divisible by $N_{T} + 1$ when establishing the lower bounds. The proofs are included  in Appendix \ref{sec:proof_lower_bound}.

%

\begin{theorem}\label{thm:reg_low_inst_dep}
For any arbitrary PS-MAB procedure with sublinear regret, i.e., $R_{T} \leq cT^{p}$ for all $T\in \mbb{N}$ for some $c >0$ and $p \in \lb 0, 1\rp$ in any PS-MABs, there exists a PS-MAB instance with at most $N_{T}$ changes and suboptimality gaps greater than $\Delta$ (i.e., $\Delta_{a,k} \geq \Delta$ for any $a \in \lb A \rb$ and $k \leq N_{T}$) such that 
%
\begin{equation} \label{eq:reg_low_inst_dep}
   R_{T} \geq \frac{\sigma^{2}}{2\Delta} \lp N_{T}+1 \rp \lp A-1 \rp \lb \log  \frac{\Delta}{8c} + \lp 1 - p \rp \log T - \log \lp N_{T} + 1 \rp \rb.
\end{equation}
which implies that $R_{T}=\Omega(AN_T\log(T/N_T))$.
\end{theorem}
%
%
%
\begin{theorem}\label{thm:reg_low_minimax}
For any arbitrary PS-MAB algorithm, there exists a PS-MAB instance with at most $N_{T}$ changes such that the regret on this bandit instance satisfies
\begin{equation} \label{eq:reg_low_minimax}
   R_{T} \geq \frac{\sigma}{27} \sqrt{\lp A - 1 \rp T \lp N_{T} + 1 \rp}.
\end{equation}
which implies that $R_{T}=\Omega(\sqrt{AN_TT})$.
\end{theorem}
%

When the number of change-points does not grow linearly with $T$, the instance-dependent bound has a logarithmic growth in $T$ which is  smaller than the $\sqrt{T}$ growth in the minimax bound. The same discrepancy exists in lower bounds for stationary bandits. Algorithms developed for PS-MABs in \cite{auer2019adswitch, besson2022efficient, wei2021non, abbasi2023new} 
attain an $\mcal{O}\lp  \sqrt{N_{T} A T \log T} \rp$ regret that matches the minimax lower bound, but not the instance dependent bound. We suspect that the instance-dependent bound is loose and unachievable; settling this conjecture remains an open question. In our proof, we construct a collection of PS-MAB isntances with identical change-points and show that the average regret over the collection must exceed the lower bounds provided. Therefore, these lower bounds can be achieved by a procedure that knows exactly  when the changes occur. 
We believe that the task of reacting to the change-points poses the fundamental bottleneck to algorithm design for PS-MABs, and hence, a lower bound analysis that simplifies the need to react to unknown change-points can be loose. 

\section{The DAB Framework}
\label{sec:DAB}
We now present the algorithm for PS-MABs that combines a stationary MAB algorithm $\mcal{B}$ with a change detector $\mcal{D}$ that is aided by extra samples from a carefully controlled forced exploration strategy. We begin by describing the nature of this forced exploration. 
Let $\tau_{k}$ be the $k^{\mrm{th}}$ detection time, with $\tau_{0}\coloneqq0$. 
Let $\alpha_{k}\in\lp0,1\rp$ be the {forced exploration frequency} for the $k^{\mrm{th}}$ interval. Then, for each $k\in[N_T]$, and for every $\lce A /\alpha_{k}\rce$ rounds in $\lbp \tau_{k - 1} + 1, \dots, \tau_{k} \rbp$, the procedure pulls all $A$ arms once in a round-robin fashion first and then follows the bandit algorithm $\mcal{B}$ afterwards. 
DAB procedures are formally presented in Algorithm \ref{alg:DAB} and illustrated in Figure \ref{fig:notation}.

Algorithm \ref{alg:DAB} requires the knowledge of the horizon $T$; this can be circumvented using the ``doubling trick'' in \cite{besson2018doublingtrickscantmultiarmed} to horizons of unknown lengths. We highlight that our DAB framework is global restart-type, where the reward history from all arms is discarded whenever a change in distribution is detected in any of the arms.
The framework is modular; $\mcal{B}$ runs on its own, while $\mcal{D}$ utilizes samples from forced exploration to detect changes and trigger restarts for $\mcal{B}$. Notice that we do not allow $\mcal{B}$ to use samples from forced exploration, but as will become clear, allowing $\mcal{D}$ to use observations of arm pulls from $\mcal{B}$ does not hinder the modularity of algorithm deign or its regret analysis.

\begin{algorithm}[!h]
\caption{Modular \textbf{D}etection \textbf{A}ugmented \textbf{B}andit (DAB) procedure}\label{alg:DAB}
\textbf{Input}: change detector $\mcal{D}$, bandit algorithm $\mcal{B}$, forced exploration frequencies $\lbp\alpha_{k}\rbp_{k=1}^{\infty}$, horizon $T$, number of arms $A$ \\
\textbf{Initialization}: last restart $\tau\leftarrow0$, $\forall\,a\in[A]$, history list for change detectors $H_{a,\mcal{D}}\leftarrow\emptyset$, history list for stationary bandit algorithm $H_{\mcal{B}}\leftarrow\emptyset$, number of intervals $k \leftarrow 1$

\begin{algorithmic}[1]
\FOR{$t=1,2,\dots,T$}
    \IF{$\lp t-\tau-1\mod{\lce\frac{A}{\alpha_{k}}\rce}\rp+1=a\in[A]$} 
        \STATE $a_t\leftarrow a$ (forced exploration (for change detector))
        \STATE Play arm $a_t$ and receive the reward $X_{a_t,t}$
    \ELSE
        \STATE $a_t\leftarrow \mcal{B}\lp H_{\mcal{B}}\rp$ (stationary bandit algorithm)
        \STATE Play arm $a_t$ and receive the reward $X_{a_t,t}$
        \STATE Add $(X_{a_t,t}, a_t)$ into the bandit history list $H_{\mcal{B}}$
    \ENDIF
    \STATE Add $X_{a_t,t}$ into the change detector history list $H_{a_t,\mcal{D}}$
    \IF{$\mcal{D}\lp H_{a_t,\mcal{D}}\rp=1$ (change detected)}
        \STATE $\tau\leftarrow t$
        \STATE $\forall\,a\in[A],\;H_{a,\mcal{D}}\leftarrow\emptyset$
        \STATE $H_{\mcal{B}}\leftarrow\emptyset$
        \STATE $k \leftarrow k + 1$
    \ENDIF
\ENDFOR
\end{algorithmic}
\end{algorithm}

\begin{figure*}[h]
    \centering
    \begin{tikzpicture}[scale=0.45]
        \draw (-21,1)-- (-5,1); 
        \node at (-21,3) {Detection Point:};
        
        \draw (-21,1.5) -- (-21,0.5);
        \node at (-21,2) {$\tau_{0}$};
        \node at (-21,0) {$0$};
        
        \draw (-18,0.5) -- (-18,1);
        \node at (-18,0) {$A$};
        \node at (-19.5,-1) {Pull each};
        \node at (-19.5,-2) {arm once};

        \draw (-10,0.5) -- (-10,1);
        \node at (-10,0) {$\lce A/\alpha_{1}\rce$~~};
        \node at (-14,-1) {Bandit};
        \node at (-14,-2) {algorithm};
        
        \draw (-7,0.5) -- (-7,1);
        \node at (-7,0) {~~$\lce A/\alpha_{1}\rce+A$};
        \node at (-8.5,-1) {Pull each};
        \node at (-8.5,-2) {arm once};
        
        \node at (-3,-1) {Bandit};
        \node at (-3,-2) {algorithm};
        \node at (-3,1) {$\cdots$};
        
        \draw (1,1.5) -- (1,0.5);
        \node at (1,2) {$\tau_{1}$};
        \node at (1,3) {Clear $H_{\mcal{B}}$ and $H_{a,\mcal{D}}$};

        \draw (-1,1) -- (7,1);
        \draw (4,0.5) -- (4,1);
        \node at (4,0) {$\tau_{1}+A$};
        \node at (2.5,-1) {Pull each};
        \node at (2.5,-2) {arm once};

        \node at (7,-1) {Bandit};
        \node at (7,-2) {algorithm};
        \node at (8,1) {$\cdots$};
        
    \end{tikzpicture}
    \caption{Illustration of the workflow of Procedure \ref{alg:DAB}.}
    \label{fig:notation}
\end{figure*}


\section{Modular Regret Analysis of DAB Procedures}
\label{sec:mod}
One of the main goals of our paper is to develop a modular regret analysis that leverages established regret properties of $\mcal{B}$ and identify properties of $\mcal{D}$ that allow a modular regret analysis that transforms an order-wise regret bound for $\mcal{B}$ and yields the counterpart for DAB. As we will illustrate, our  analysis can be applied to obtain regret bounds for various combinations of change detectors and bandit algorithms.
 Such an analysis is largely absent in prior work.

%

\subsection{Requirements for Stationary Bandit Algorithms and Change Detectors} \label{sec:MAB_CD_Req}
We require $\mcal{B}$ to satisfy the following property.

\begin{property}[stationary bandit algorithm regret]\label{proper:regret}
For the stationary bandit algorithm $\mcal{B}$ on a  stationary bandit instance with $A$ arms and suboptimality gaps $\lbp\Delta_{a}\rbp_{a=1}^{A}$ where $\Delta_{a} \leq C$ for any $a\in[A]$, its minimax regret upper bound over $T$ rounds is  $R_{\mcal{B}}\lp T \rp$ that is concave and increases sublinearly with $T$, specifically as {$\mcal{O}\lp \sqrt{T}\rp$} up to polylog factors. 
\end{property} 
This property holds for many well-known bandit algorithms. For example, the regret upper bound for UCB, which is independent of the suboptimality gaps, is $8\sqrt{A T \log \lp T \rp} + \mcal{O} \lp 1 \rp$  \cite[Theorem 7.2]{lattimore2020bandit} and satisfies Property~\ref{proper:regret}. We 
choose to limit the minimax regret rate for $\mcal{B}$ because that bound applies to stationary bandit instances with arbitrary sub-optimality gaps across change-points in an PS-MAB instance.


Next, we turn to the requirements for the change detector $\mcal{D}$. The goal is to detect changes as soon as possible while not raising false alarms too often over the horizon. Taking cues from the regret analysis in \cite{besson2022efficient}, if the change detector gets falsely triggered or if it detects a change too late, the samples for detecting the next change-point may be insufficient, making the change detector unable to detect the next change. When any of the changes remain undetected over the entire horizon, which is defined as a \emph{missed detection} event, the worst-case regret bound is linear \cite{besson2022efficient}. Thus, we control the probability of missed detection. Because false alarm events and late detection events could possibly lead to missed detection events, we additionally  ensure that these events happen with a small probability. Finally, since pulling suboptimal arms during the detection delay will also lead to linear regret, the threshold for detection delay, referred to as the \emph{latency}, is kept small.

Before formally laying out the properties that $\mcal{D}$ should possess, we formulate the general QCD problem associated with our analysis: Let $\{X_n: n \in \mbb{N}\}$ be a sequence of independent random variables observed sequentially by the detector.  When the change-point occurs at $\nu\in\mbb{N}$, 
\begin{align}
    X_{n}\sim\begin{dcases}
    f_{0},\;n < \nu\\
    f_{1},\;n\geq \nu
    \end{dcases}.\label{eq:sample_distr}
\end{align}
In other words, before the change-point $\nu$, the stochastic samples follow the pre-change distribution with density $f_{0}$. The remaining samples follow the post-change distribution with density $f_{1}$. Additionally, $\Pr_{\nu}$ denotes the probability measure under which the change-point occurs at $\nu\in\mbb{N}$, while $\Pr_{\infty}$ denotes the probability measure under which there is no change-point ($\nu=\infty$). We assume that the densities $f_{0}$ and $f_{1}$ are $\sigma^{2}$-sub-Gaussian with  mean $\mu_{0}$ and $\mu_{1}$, and let $\Delta_{\mathrm{c}} \coloneqq \lba\mu_{0}-\mu_{1}\rba$. We also assume that the change detector only knows $\sigma$, but is agnostic to the actual densities $f_{0}$ and $f_{1}$. Let the change-point be deterministic and unknown to the change detector.  

Over a finite horizon $M\in\mbb{N}$, the detector samples the random variables $X_{1},\dots,X_{M}$ sequentially. Every causal change detector $\mcal{D}$ can be associated with a stopping time $\tau$, at which the detector declares a change. Because $f_{0}$ is unknown to $\mcal{D}$, we need to guarantee that there are enough samples for $\mcal{D}$ to learn sufficient information about the pre-change distribution. Hence, we assume that there exists a \emph{pre-change window} $m$ in which the change-point does not occur (i.e., $\nu>m$). 

We define the \emph{latency} $d$ associated with a change detector $\mathcal{D}$ as the length of time post-change within which a change is declared with a probability $1-\delta_{\mathrm{D}}$, i.e., 
\begin{equation}\label{eq:latency_def}
d \coloneqq\inf\{ t\in[M]:\mathbb{P}_{\nu}(\tau\geq\nu+t)\leq\delta_{\mathrm{D}}, \forall\,\nu\in\lbp m+1,\dots,M-t\rbp\}.   
\end{equation}
The latency $d$ can be thought of a high probability version of the delay as opposed to the average delay that is typically used as a metric in QCD problems as in \cite{lorden1971procedures, moustakides1986optimal, lai1998information}.
A good change detector $\mathcal{D}$ seeks to minimize $d$ (at performance level $\delta_{\mathrm{D}}$) with low false alarm probability over the horizon $M$, i.e., 
\begin{equation} \label{eq:PFbound}
\mathbb{P}_{\infty} (\tau \leq M) \leq \delta_{\mathrm{F}}, \textup{~~with~} \delta_{\mathrm{F}} \in (0, 1).
\end{equation}
The change detection problem is characterized by the horizon of interest $M$ and the mean-shift $\Delta_{\mrm{c}}$. ${\mcal D}$ then defines a stopping rule to yield the pre-change window length $m$, given the required performance levels \textup{$\delta_{\mrm{F}}, \delta_{\mrm{D}} \in\lp0,1\rp$}, which in turn yields a latency $d$. While one would ideally like $\mcal{D}$ to use fewer samples $m, d$, there is a trade-off, however; smaller $d$ requires a larger $m$, i.e., more pre-change samples are required to flag a change with low latency. As the proofs of the regret bound results in Theorem~\ref{thm:det_reg} and Corollary~\ref{cor:order_reg} reveal, the growth rate of $m$ and $d$ for a good change detector $\mcal{D}$ must satisfy the following property to guarantee optimal regret growth.
\begin{property}[change detector latency]
\label{proper:cd}
Consider the change detection problem for the observation model in \eqref{eq:sample_distr} with mean shift $\Delta_{\mrm{c}}$ over horizon $M$. Furthermore, consider a change detector $\mcal{D}$ with performance levels \textup{$\delta_{\mrm{F}}, \delta_{\mrm{D}} \in\lp0,1\rp$}, with stopping time $\tau$ and pre-change window $m$ chosen to satisfy \eqref{eq:PFbound}. Then the pre-change window $m$ and latency $d$ defined in \eqref{eq:latency_def}  should satisfy the following properties: 
\begin{itemize}
    \item[(i)]  $d$ and $m$ should be decreasing with $\Delta_{\mathrm{c}}$ and increasing with $M$,
\item[(ii)] $m + d  \lesssim  \log M + \log (1/\delta_{\mathrm{F}}) + \log (1/\delta_{\mathrm{D}})$.
\end{itemize}
\end{property}

Notice that $\Delta_{\mrm{c}}$ is a measure of how discernible the changes are. The larger $\Delta_{\mrm{c}}$ is, the easier it should become to detect with a reasonable change detector, requiring lower values of $m$ and $d$  for guaranteed performance levels defined by \textup{$\delta_{\mrm{F}}, \delta_{\mrm{D}}$}. Larger horizons must impose greater chances of false alarms and delay, and therefore again, the $m$ and $d$ should grow with the horizon $M$. Furthermore, the change detection must occur sufficiently fast and must not dominate the regret growth of the stationary bandit algorithm. As our analysis will show, the logarithmic growth of $m+d$ in part (ii) of Property~\ref{proper:cd} yields order-optimal regret bounds for DAB procedures.

\subsection{Modularized Regret Analysis} \label{sec:MAnal}
Having outlined the properties we want $\mcal{B}$ and $\mcal{D}$ to satisfy, we now present the frequency of forced exploration and then the regret upper bounds for DAB procedures with said frequency. To specify that frequency, notice that 
$\mcal{D}$ that satisfies Property \ref{proper:cd} needs at least $m$ pre-change samples and $d$ post-change samples to detect changes quickly with high probability. For the DAB procedures, however, each arm is not pulled every time step. 
With a forced exploration frequency $\alpha_{k}$, $\mcal{D}$ is guaranteed to obtain one sample from each arm every $\lce A/\alpha_{k}\rce$ rounds. 
Then, we define the latency and the pre-change window for a DAB procedure at the $k^{\mrm{th}}$ change-point as:
\begin{align}
    m_{k}&\coloneqq\lce A/\alpha_{k} \rce m\lp\underline{\Delta}_{\mrm{c}}, T\rp,\label{eq:pre-window}\\
    d_{k}&\coloneqq\lce A/\alpha_{k} \rce d\lp\underline{\Delta}_{\mrm{c}}, T\rp. \label{eq:latency}
\end{align}
where $d$ and $m$ are explicitly written as functions of $\underline{\Delta}_{\mrm{c}}$, the minimum change-gap. Define $d_{0}\coloneqq0$ for notational convenience. In the definition of $m_k$ and $d_k$, the horizon is assumed to be $T$, rather than the rounds remaining after the latest detection (which is upper-bounded by $T$). This is justified by Property~\ref{proper:cd} (i), which says that $m$ and $d$ are increasing with $M$.

With Properties \ref{proper:regret} and \ref{proper:cd} being satisfied, we have the following result that characterizes the modular regret upper bound for DAB procedures with $\mcal{B}$ and $\mcal{D}$ with the specified forced exploration frequency. The proof is given in the Appendix~\ref{sec:thm1}.
\begin{theorem}[modular regret upper bound for DAB procedures]\label{thm:det_reg}
Consider a piecewise stationary bandit environment with minimum change-gap $\underline{\Delta
}_{\mrm{c}}$. Furthermore, consider a DAB procedure (Procedure~\ref{alg:DAB}) using a change detector $\mcal{D}$ with parameters \textup{$\delta_{\mrm{F}}$ and $\delta_{\mrm{D}}$}, stationary bandit algorithm $\mcal{B}$ (with suboptimality gap upper bound $C$), and forced exploration frequencies $\lp \alpha_{k} \rp_{k = 1}^{N_{T} + 1}$
Suppose further that the following condition holds:
\begin{equation} \label{eq:cond1}
d_{k - 1} + m_{k} \leq \nu_{k} - \nu_{k - 1}, ~\emph{for~all~} k \in [N_T]
\end{equation}
where $m_{k}$ and $d_k$ are as defined in \eqref{eq:pre-window} and \eqref{eq:latency}, respectively. Then, the regret
is upper bounded as follows:
\begin{align}
\begin{aligned}
    R_{T}& \leq C T A\lp N_{T} + 1\rp\delta_{\mrm{F}} + C T N_{T} \delta_{\mrm{D}} + \lp N_{T} + 1 \rp R_{\mcal{B}}\lp \frac{T}{N_{T} + 1} \rp \\
    &\quad+ C \sum_{k=1}^{N_{T}} d_{k} + C\lb \bar{\alpha} T + \lp N_T + 1 \rp A \rb \label{eq:det_reg}.
\end{aligned}
\end{align}
where $\bar{\alpha} \coloneqq \max_{k = 1, \dots, N_{T} + 1} \alpha_{k}$.
\end{theorem}

Condition \eqref{eq:cond1} in Theorem \ref{thm:det_reg} guarantees that the $k^{\text{th}}$ change-point will not happen during the pre-change window, given that the $\lp k - 1 \rp^{\text{th}}$ change is detected within $d_{k}$. A similar condition (see \cite[Assumption 4]{besson2022efficient}) is also imposed in the regret analysis for the GLR-klUCB algorithm.  Without this condition, a careful analysis of the regret would require bounding more precisely the effect of missing a change on the detection of subsequent changes, which can be a challenging task. We expound further on the change-point separation requirement in Theorem \ref{thm:det_reg} in the discussion surrounding Corollary \ref{cor:order_reg_2}.
There are five different components that contribute to the regret bound in Theorem \ref{thm:det_reg}. The first term $C T A\lp N_{T} + 1\rp\delta_{\mrm{F}}$ stems from false alarm events, in which the number of intervals is $N_{T}+1$, and the probability of false alarm on each arm during each interval is $\delta_{\mrm{F}}$. Each false alarm event then leads to linear regret $C T$ in the analysis. 
The second term $C T N_{T} \delta_{\mrm{D}}$ results from late detection events, in which the number of change-points is $N_{T}$ and the probability that all arms fail to detect within the latency is $\delta_{\mrm{D}}$. The late detection event also leads to linear regret $C T$ in the analysis. 
The third term $\lp N_{T} + 1\rp R_{\mcal{B}}\lp T / \lp N_{T} + 1 \rp \rp$ results from the stationary bandit algorithm. The intuition is that the regret during each interval is upper bounded by $R_{\mcal{B}}\lp \nu_{k} - \nu_{k - 1} \rp$, and the summation of these regret upper bounds is maximized when $\nu_{k} - \nu_{k - 1}$ is approximately $T/\lp N_{T} + 1 \rp$ for each $k$. 
The fourth term $ C \sum_{k=1}^{N_{T}} d_{k}$ represents the regret during delays, for which the latency $d_{k}$ associated with change-point $\nu_k$ is smaller than $d_{k}$. 
The forced exploration leads to the fifth term $C \bar{\alpha} T + C \lp N_T + 1 \rp A $, for which, over the finite horizon $T$, the procedure executes forced exploration for approximately at most $\bar{\alpha} T$ rounds, and $C \lp N_T + 1 \rp A$ bounds the round-off errors.

In Procedure \ref{alg:DAB}, we do not allow the stationary bandit algorithm to observe the samples acquired from forced exploration. If the stationary bandit algorithm were to have access to all samples, as is the case in \cite{besson2022efficient}, $R_{\mcal{B}}$ could not be plugged directly into our regret analysis, thereby breaking the modularity of the regret upper bound (see also step $(a)$ in \eqref{eq:bound_inter} in the proof of Theorem \ref{thm:det_reg}). On the other hand, 
letting the change detector having access to samples obtained from stationary bandit algorithms does not affect the modularity.

\begin{remark}
Although the bound in Theorem \ref{thm:det_reg} appears to be linear with respect to $T$, we stress that we can set $\delta_{\mrm{F}}, \delta_{\mrm{D}}$, and $\lbp\alpha_{k} \rbp$ in a manner to make the regret upper bound in \eqref{eq:det_reg} a sublinear function of $T$. These choices of $\delta_{\mrm{F}}$, $\delta_{\mrm{D}}$ and $\lbp\alpha_{k}\rbp$ are described in Corollary~\ref{cor:order_reg} in the next subsection.
\end{remark} 


\subsection{Application to Various Combinations of Change Detectors and Stationary Bandit Algorithms} \label{sec:DAB_cor}

Theorem \ref{thm:det_reg} allows us to study regret upper bounds for DAB procedures that combine different stationary bandit algorithms with different change detectors. Consider any stationary bandit algorithm for which the regret satisfies Property~\ref{proper:regret} and  scales with $T$ as at most $\sqrt{T \log(T)}$. Examples include UCB \cite{lattimore2020bandit} and klUCB \cite{cappe2013kullback}, for which we have the following (instance-independent) regret upper bounds:
\begin{align}\label{eq:UCBs}
    &R_{\mrm{UCB}}\lp T \rp = 8\sqrt{\sigma^{2} A T \log \lp T \rp} + \mcal{O} \lp 1 \rp,\\
    &R_{\mrm{klUCB}}\lp T \rp \coloneqq 2\sqrt{2\sigma^{2}AT\log (T)} + \mcal{O} \lp 1 \rp.
\end{align}

Next, we consider change detectors that satisfy Property~\ref{proper:cd}. The first candidate we study is a generalized likelihood ratio (GLR) based QCD test designed for sub-Gaussian observation statistics, which is similar to the GLR QCD test for sub-Bernoulli statistics used in \cite{besson2022efficient}.
For any desirable false alarm probability $\delta_{\mrm{F}}\in\lp0,1\rp$, define 
\begin{equation}\label{eq:GLR_beta}
    \beta\lp n,\delta_{\mrm{F}}\rp\coloneqq6\log\lp1+\log\lp n\rp\rp+\frac{5}{2}\log\lp\frac{4n^{3/2}}{\delta_{\mrm{F}}}\rp+11.
\end{equation}
The stopping time of the GLR test is given by
\begin{align}
    \tau\coloneqq\inf\lbp n\in\mbb{N}: G_{n}\geq\beta\lp n,\delta_{\mrm{F}}\rp\rbp\label{eq:GLR_stop}
\end{align}
where the GLR statistics $G_{n}$ is
\begin{align}
    G_{n}\coloneqq\sup_{s\in \lb n\rb}\log\lp\frac{\sup_{\theta_{0},\theta_{1}\in\mbb{R}}\sup_{\theta_{1}\in\mbb{R}}\prod_{i=1}^{s}f_{\theta_{0}}\lp X_{i}\rp\prod_{i=s+1}^{n}f_{\theta_{1}}\lp X_{i}\rp}{\sup_{\theta\in\mbb{R}}\prod_{i=1}^{n}f_{\theta}\lp X_{i}\rp}\rp\label{eq:GLR_stat}
\end{align}
in which $f_{\theta}$ is the density of a Gaussian random variable with mean $\theta \sigma^{2}$ and variance $\sigma^{2}$. 

The GLR test can be considered as a generalization of the CuSum test \cite{vvv_qcd_overview}. A well-known alternative to the CuSum test for QCD problems is Shiryaev-Roberts (SR) test \cite{vvv_qcd_overview}, and we can construct a generalization of this test, which we call the generalized Shiryaev-Roberts (GSR) test, which is characterized by the stopping rule,
\begin{align}
    \tau\coloneqq\inf\lbp n\in\mbb{N}: \log W_{n} \geq \beta \lp n,\delta_{\mrm{F}}\rp + \log n \rbp\label{eq:GSR_stop}
\end{align}
where the GSR statistic $W_{n}$ is given by
\begin{align}
W_{n}\coloneqq \frac{1}{n}\sum_{s=1}^{n}\lp\frac{\sup_{\theta_{0}\in\mbb{R}}\sup_{\theta_{1}\in\mbb{R}}\prod_{i=1}^{s}f_{\theta_{0}}\lp X_{i}\rp\prod_{i=s+1}^{n}f_{\theta_{1}}\lp X_{i}\rp}{\sup_{\theta\in\mbb{R}}\prod_{i=1}^{n}f_{\theta}\lp X_{i}\rp}\rp.
\label{eq:GSR_stat}
\end{align}
%


In Theorem 2 of \cite{huang2025sequential},  GLR and GSR change detectors have been shown to satisfy
\begin{equation} \label{eq:GLRGSRp2}
d  \lesssim  \log M + \log (1/\delta_{\mathrm{F}}) + \log (1/\delta_{\mathrm{D}}).
\end{equation}
Given this characterization of the GLR and the GSR change detectors, Theorem \ref{thm:det_reg} now allows us to deduce the regret upper bound on DAB procedures that combine efficient stationary bandit algorithms, such as UCB and klUCB, with efficient change detectors, such as GLR and GSR change detectors.

Before we present our formal regret analysis, we take a  detour and search for parameters for  DAB procedures that yield a regret bound that is $\mcal{O}\lp\sqrt{A N_T T \log(T)}\rp$ according to Theorem \ref{thm:det_reg}, with a \emph{time-uniform} exploration policy, i.e., $\alpha_k =\alpha$, for all $k$.  This detour will serve as a prelude to the parameter choices made in Corollary \ref{cor:order_reg} in which the exploration frequency is non-uniform over time. Our choice of the $\sqrt{A N_T T \log(T)}$ rate is motivated by \cite{besson2022efficient}, where it is shown that this rate is achieved for a specific GLR-klUCB procedure, albeit with an analysis that is not modularized as in the current work. In addition, the rate of $\sqrt{A N_T T \log(T)}$ is also provably almost optimal, given the regret lower bound of $\Omega (\sqrt{A N_T T})$ in Section \ref{sec:prob} and in prior works, \cite{garivier2011upper, auer2019adswitch, seznec2020single}.

Consider a stationary bandit algorithm with regret bounded as $\mathcal{O}(\sqrt{A T\log(T)})$. Further, suppose that $\delta_{\mrm{F}} = \delta_{\mrm{D}} = T^{-\gamma}$ for some $\gamma > 1$. Then from \eqref{eq:GLRGSRp2}, we have that $d \lesssim \log(T)$ for the GLR and GSR change detectors, and for that matter, any change detector that satisfies Property~\ref{proper:cd}. Under the time-uniform exploration policy, we have from \eqref{eq:latency} that 
\[
d_k = \lce \frac{A}{\alpha_k} \rce d = \lce \frac{A}{\alpha} \rce d\lesssim \frac{A}{\alpha} \log (T).
\]
The terms in the regret upper bound of Theorem 1 then behave as follows for large $T$.
\begin{subequations}
\begin{align}
    C T A\lp N_{T} + 1\rp\delta_{\mrm{F}}
    &\simeq A T^{1-\gamma} N_T,
    \\
    C T N_{T} \delta_{\mrm{D}} 
    &\simeq T^{1-\gamma} N_T,
    \\
     \lp N_{T} + 1 \rp R_{\mcal{B}}\lp \frac{T}{N_{T} + 1} \rp 
    &\lesssim N_T \sqrt{A \frac{T}{N_T} \log\lp \frac{T}{N_T}\rp } = \sqrt{A {T}{N_T} \log\lp \frac{T}{N_T}\rp },\label{eq:11}
    \\
    C \bar{\alpha} T + C \lp N_T + 1 \rp A &\simeq \alpha T + A N_T, 
    \\
    C\sum_{k=1}^{N_T} d_k &\lesssim N_T \frac{A}{\alpha} \log(T).
\end{align}
\label{eq:thm1.const.alpha}
\end{subequations}
Can we choose $\alpha$ to achieve an overall $\mathcal{O}(\sqrt{A T N_T \log(T)})$ regret? 
Among the five terms, the first two do not matter as they go to zero with growing $T$, since  $N_T \leq T$ and $\gamma >1$.  The third term already satisfies the desired growth rate. The fourth term must satisfy
\begin{align}
    \alpha T \lesssim \sqrt{A N_T T \log (T)} \implies \alpha \lesssim \sqrt{ A N_T \log(T)/T}.
\end{align}
The fifth term must satisfy
\begin{align}
    N_T \frac{A}{\alpha} \log(T) \lesssim \sqrt{A N_T T \log (T)} 
    \implies \alpha \gtrsim \sqrt{ A N_T \log(T)/T}.
\end{align}
Interestingly, the last two conditions identify the required asymptotic growth rate of $\alpha$, i.e., that $\alpha \simeq \sqrt{A N_T\log(T)/T }$. However, implementation of such an algorithm requires the knowledge of $N_T$. To circumvent the same, a simple modification comes to the rescue, which is to replace the constant exploration policy with a non-uniform one, where $\alpha_k \simeq \sqrt{A k \log(T)/T}$ for the $k^{\text{th}}$ interval. 
The regret analysis for the non-uniform exploration policy follows along the same lines, and is formally encapsulated in the following result.

\begin{corollary}\label{cor:order_reg}
Consider  Procedure~\ref{alg:DAB} combining a stationary bandit algorithm $\mcal{B}$ with 
$R_{\mathcal{B}} = \mathcal{O}(\sqrt{A T\log(T)})$, and with a change detector $\mcal{D}$ that satisfies Property~\ref{proper:cd}, on a piecewise-stationary MAB problem. Suppose  \textup{$\delta_{\mrm{F}}=\delta_{\mrm{D}}=T^{-\gamma}$} for some $\gamma>1$, and $\alpha_{k} = \alpha_{0}\sqrt{k A \log\lp T\rp/T}$.
Then, if condition \eqref{eq:cond1} holds, $R_T \lesssim  \sqrt{A N_{T} T \log \lp T \rp} $.
\end{corollary}
%
\begin{proof}
The proof steps remain the same as in the uniform exploration case in \eqref{eq:thm1.const.alpha}, with the following two exceptions. First, we have
\begin{align}
    C \bar{\alpha} T + C \lp N_T + 1 \rp A \lesssim \sqrt{A N_T T \log(T)}.
\end{align}
Second, since $d_k$ varies as $Ad/\alpha_k$,  we obtain
\begin{align}
    C \sum_{k=1}^{N_{T}} d_{k} 
    \lesssim \sum_{k = 1}^{N_{T}} A \log(T) \sqrt{\frac{T}{A k \log \lp T \rp}} 
    = \sqrt{A N_{T} T \log \lp T \rp}.
\end{align}
This calculation yields the desired upper bound on the overall regret $R_T$.
\end{proof}
An important message we would like to emphasize is the necessity of forced exploration. 
Most good stationary bandit algorithms (such as klUCB and UCB) pull suboptimal arms at a logarithmic rate. Although the latencies for the GLR and GSR tests are $\mcal{O} \lp \log \lp T \rp \rp$, the DAB procedures constructed from these tests would need $\mcal{O} \lp T \rp$ time steps to obtain $\mcal{O} \lp \log \lp T \rp \rp$ samples for suboptimal arms \emph{without forced exploration}. This would make the delay $\mcal{O} \lp T \rp$ and thus lead to linear regret. To achieve order-optimal regret, we set $\alpha_{k} = \alpha_{0} \sqrt{k A \log \lp T \rp / T}$. This choice of $\alpha_{k}$ guarantees that the number of pulls from a suboptimal arm is $\Omega \lp \sqrt{T \log\lp T \rp} \rp$, making the delay sublinear. In addition, the regret due to delay and the regret due to forced exploration match the order of the regret due to stationary bandit algorithm, making the overall regret order-optimal (with an extra $\sqrt{\log T}$ factor). Finally, it is apparent from this discussion that a DAB procedure should not need forced exploration when the stationary bandit algorithm pulls suboptimal arms at a rate of the order $\sqrt{T \log \lp T \rp}$. The necessity of forced exploration is further corroborated by our experiments in Section \ref{sec:exp}. Consequently, one could deliberately make the stationary bandit algorithm learn at a slower rate to gather more information about the suboptimal arms. However, such a modification to the bandit algorithm is unnecessary and runs counter to the plug-and-play modularization of the analysis of the regret of DAB procedures, which is the main goal of this paper.

We end this analysis by delving deeper into the change-point separation condition in \eqref{eq:cond1}. Given the increasing nature of $\alpha_k$, we deduce that the right-hand side of  \eqref{eq:cond1} is bounded above as
\begin{align}
    d_{1} + m_{1} = \lceil A/\alpha_1\rceil \lp d\lp\underline{\Delta}_{\mrm{c}}, T\rp + m\lp\underline{\Delta}_{\mrm{c}}, T\rp\rp \lesssim \sqrt{T \log T}.
\end{align}
Let the minimum change-point separation be defined as
\begin{align}
    L_T \coloneqq \lp\min_{k\in [N_T]} \nu_k -\nu_{k-1}\rp.
\end{align}
If we assume $L_T \gtrsim \sqrt{T \log (T)}$, e.g., say $L_T = \lp \sqrt{T \log (T)}\rp^{1+\epsilon}$, for some small $\epsilon > 0$, then condition \eqref{eq:cond1} will hold asymptotically, i.e., for all sufficiently large $T$.
Thus, under the assumption, we deduce the following result.

\begin{corollary}\label{cor:order_reg_2}
Consider  Procedure~\ref{alg:DAB} combining a stationary bandit algorithm with 
$R_{\mathcal{B}} = \mathcal{O}(\sqrt{A T\log(T)})$, and with a change detector that satisfies Property~\ref{proper:cd}, on a piecewise-stationary MAB problem. Suppose  \textup{$\delta_{\mrm{F}}=\delta_{\mrm{D}}=T^{-\gamma}$} for some $\gamma>1$, and $\alpha_{k} = \alpha_{0}\sqrt{k A \log\lp T\rp/T}$.
Then, if $L_T \gtrsim \sqrt{T \log (T)}$, $R_T \lesssim  \sqrt{A N_{T} T \log \lp T \rp} $.
\end{corollary}
%
The constraint $L_T \gtrsim \sqrt{T \log (T)}$ restricts the number of change-points $N_{T}$,
\begin{align}
    N_{T} \leq \frac{T}{L_{T}} \lesssim \sqrt{T/\log(T)}.
\end{align}
We argue why this is an interesting regime for a PS-MAB environment. Notice that if $N_T \simeq T$ and even if a bandit algorithm takes just a single sample to readjust to the new stationary environment, the regret grows $\mcal{O}(N_T) = \mcal{O}(T)$. In other words, such PS-MAB environments pose an impossible challenge that cannot be surmounted without altering the regret definition in \eqref{eq:regret}. If $N_T \simeq \mcal{O}(1)$, then the intervals over which the environment is stationary are far too long, and this would potentially allow even a stationary bandit algorithm enough time to readjust to a changed reward distribution through its in-built exploration strategy. We posit that  $1 \lesssim N_T \lesssim T$ \emph{is} the interesting regime for learning in PS settings. The asymptotic analysis of Corollary \ref{cor:order_reg_2} has $N_T \simeq L_T \simeq \sqrt{T}$ as the boundary case, i.e., minimum (resp. maximum) allowable rate for $L_T$ (resp. $N_T$) is $\sqrt{T}$. Our analysis is made under the premise that \emph{all} change-points are detected with high probability. In practice, one expects a change-detector to possibly make some errors in detection of change-points, but ``catch up'' within a few change-points. We observe this in our simulations in Section \ref{sec:exp}. When $N_T \gtrsim \sqrt{T}$, missing a constant number of change points results in $\mcal{O}(\sqrt{T})$ regret from the intervals corresponding to the missed change-points, provided the stationary bandit algorithm is fast enough to adapt. As a result, we expect the overall regret rate to hold, even when one violates the separation assumption made in our analysis. A precise characterization of regret with shorter intervals requires one to analyze the behavior of the change detection statistic after multiple possible changes and bound the probability of ``cascades'', where missing one change-point leads to continued errors in change-detection downstream. Such a task is left for future work.



\section{Experimental Study}
\label{sec:exp}
In our experiments, we performed numerical simulations on synthetic data. These simulations aimed to compare the efficacy of our approach by combining different change detectors and stationary bandit algorithms. We also benchmarked our proposed DAB procedures against representative methods from prior works. Finally, we examined the detection performance of our approach, along with an analysis of the implications of condition \eqref{eq:cond1}.
\subsection{Experimental Benchmark}
    
We designed a PS-MAB environment with \emph{stochastic} change-points, where the intervals between change-points are i.i.d.\ geometric random variables. Such a stochastic model naturally introduces variability in the placement of change-points, enabling  evaluation across a broad range of scenarios. This helps assess performance in environments where change-points occur unpredictably as one would expect in practice.
Specifically, we simulated environments with number of arms $A=5$, the horizon $T=100000$, and the intervals between change-points are
i.i.d. geometric with parameter $\rho$, where $\rho=T^{-\xi}$, for $\xi \in \lbp 0.3,0.4,0.5,0.6,0.7,0.8 \rbp$. The expected number of change-points is $\rho T = T^{1-\xi}$, which is sublinear in $T$ for all the values of $\xi$'s considered. The regret is averaged over $2000$ independent trials. In addition, rewards are in $[0,1]$, while the magnitude of the change is uniformly sampled from $[0.1,0.4]$.

\subsection{Algorithms and Parameters}
    To demonstrate the modularity of our method, different change detectors are combined with various bandit algorithms. For the change detectors, we utilize both the Bernoulli and Gaussian variants of the GSR and GLR tests. The stationary bandit algorithms  include UCB, MOSS \cite{audibert2009minimax} and klUCB.
    Among prior approaches, we consider both methods designed for non-stationary settings and those developed for stationary environments. Among the non-stationary methods, we include M-klUCB \cite{cao2019nearly} and CUSUM-klUCB \cite{liu2018change}, which require prior knowledge of the non-stationarity, as well as GLR-klUCB \cite{besson2022efficient} with global restarts (in both Bernoulli and Gaussian variants), MASTER \cite{wei2021non}, and our own approach, all of which operate without such knowledge. For stationary baselines, we include klUCB \cite{garivier2011kl} and UCB \cite{auer2002finite}, both without change detection. We exclude AdSwitch \cite{auer2019adswitch} due to its significant computational overhead and its demonstrated poor empirical performance (even worse than stationary klUCB) \cite{besson2022efficient}. All methods are parameterized following their original works. Finally, we emphasize a key distinction: in \cite{besson2022efficient}, GLR-klUCB employs a shared history between the klUCB algorithm and the GLR test, whereas in our DAB procedure the klUCB component does not use the forced exploration samples. Regarding the forced exploration frequencies, \cite{besson2022efficient} suggests scaling the exploration frequency as 
    $\alpha_k=0.1 \sqrt{k A \log(T)/T}$. In the DAB procedures, choosing the exploration frequency parameter $\alpha_0=0.05$ yielded better performance. To this end, we inspect the effect of changes in $\alpha_0$ in section \ref{subsec:exp_res}.
\subsection{Practical Tuning of QCD Tests}
    Our simulations show that the threshold function provided in \eqref{eq:GLR_beta} is  conservative. Thereupon, to mitigate this issue in practice as done in \cite{besson2022efficient}, we set $\beta(n,\delta_{\mrm{F}})=\log{\left(4n^{3/2}/\delta_{\mrm{F}}\right)}$ and perform auxiliary down-sampling. Specifically, the GLR and GSR tests are computationally intensive, and for practical implementation we conduct the GLR test every 10 time steps and examine every 5 observations for a potential change-point. In contrast, we apply the GSR test every 10 time steps only, as further down-sampling is not possible for the observations. Finally, regarding the selection of the parameter $\gamma$ for $\delta_{\mrm{F}}$ and $\delta_{\mrm{D}}$,  we found that choosing $\gamma > 1$ for our DAB procedures tends to be conservative. Thus, to ensure a fair comparison across methods like GLR-klUCB, we set $\gamma = 1/2$ in our experiments.

\subsection{Experimental Results}
\label{subsec:exp_res}
\begin{figure*}[h]
\begin{center}
    \begin{subfigure}[b]{0.49\linewidth}
       \includegraphics[draft=false,width=\linewidth]{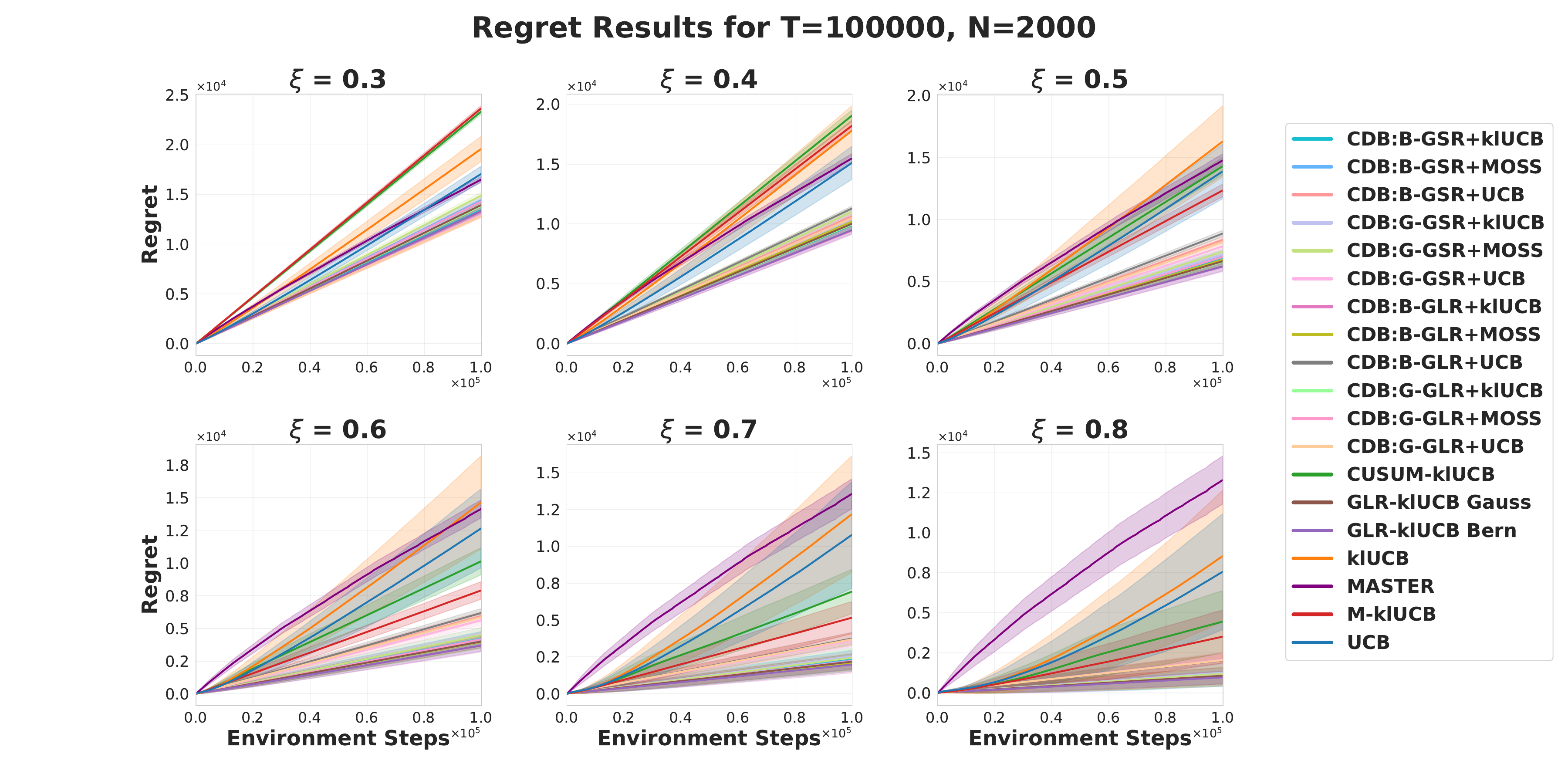}
        \caption{Comparison of all methods of Table \ref{table:final_regret}.}
        \label{fig:all_methods}
    \end{subfigure}
    \begin{subfigure}[b]{0.49\linewidth}
        \includegraphics[draft=false,width=\linewidth]{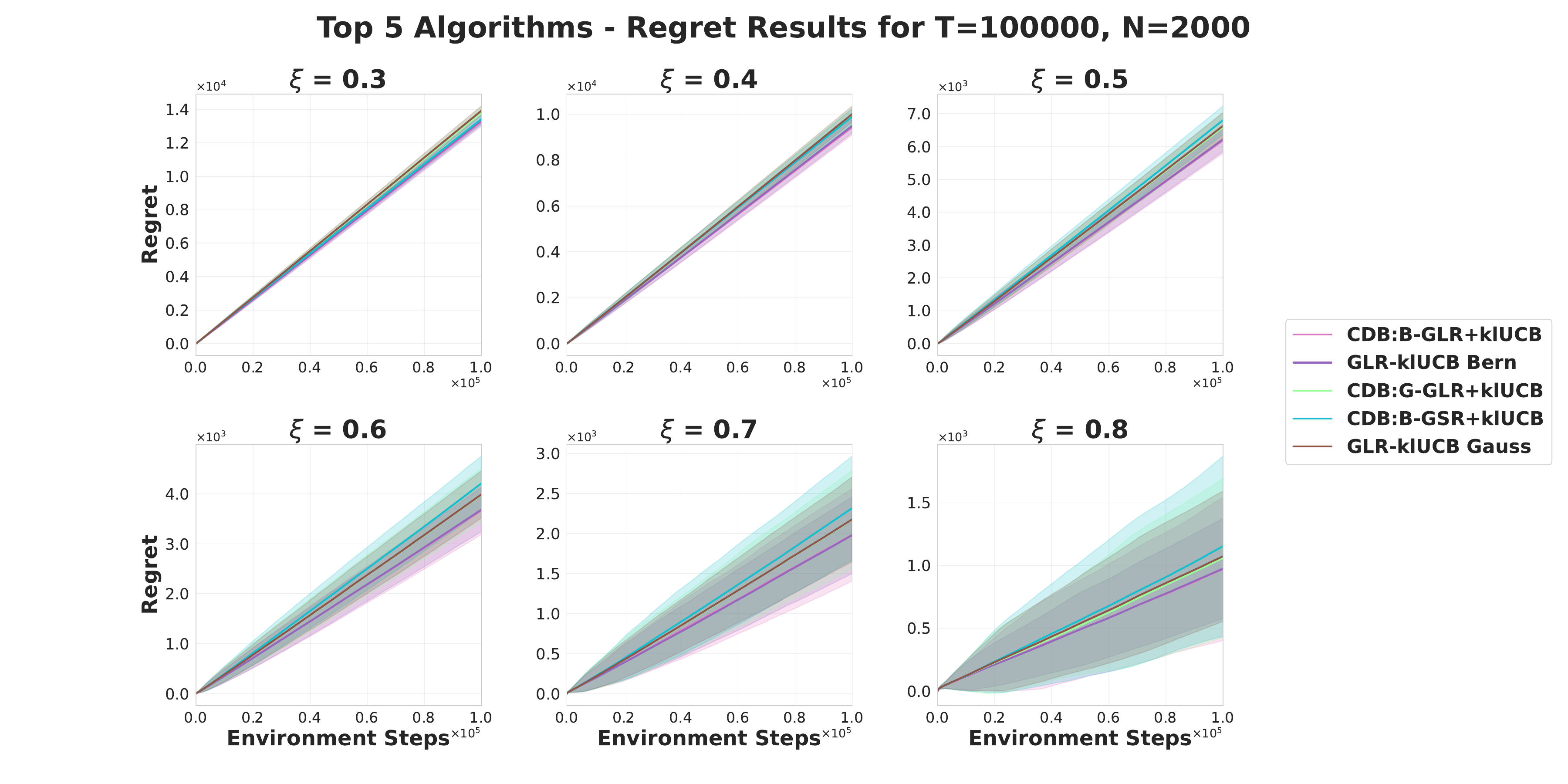}
        \caption{Comparison of five best performing methods in Figure \ref{fig:all_methods}.}
        \label{fig:best_methods}
    \end{subfigure}
\end{center}
\caption{Regret plots versus the time steps for $T=100000$, averaged over $2000$ runs.}
\label{fig:reg_results}
\end{figure*}

\begin{table}[h]
\caption{Final Regret Results (Mean $\pm$ Std), $T=100000$, 2000 Runs.}
\label{table:final_regret}
\scriptsize
\setlength{\tabcolsep}{3pt}
\begin{center}
\begin{tabular}{l|cccccc}
\hline
\multicolumn{1}{l|}{\bf Algorithm} 
& \multicolumn{1}{c}{\bf $\xi=0.3$} 
& \multicolumn{1}{c}{\bf $\xi=0.4$} 
& \multicolumn{1}{c}{\bf $\xi=0.5$} 
& \multicolumn{1}{c}{\bf $\xi=0.6$} 
& \multicolumn{1}{c}{\bf $\xi=0.7$} 
& \multicolumn{1}{c}{\bf $\xi=0.8$} \\
\hline
{UCB} & 17043.56 $\pm$ 775.01 & 15089.05 $\pm$ 1385.42 & 13876.87 $\pm$ 2206.86 & 12634.53 $\pm$ 3021.33 & 10768.53 $\pm$ 3605.60 & 7563.37 $\pm$ 3613.44 \\
{klUCB} & 19556.07 $\pm$ 1269.81 & 17783.91 $\pm$ 2078.54 & 16292.08 $\pm$ 2873.23 & 14625.86 $\pm$ 3549.95 & 12180.54 $\pm$ 3947.70 & 8540.80 $\pm$ 4101.70 \\
{CUSUM-klUCB} & 23309.39 $\pm$ 269.66 & 19029.19 $\pm$ 475.09 & 14294.10 $\pm$ 709.21 & 10119.63 $\pm$ 1029.34 & 6922.56 $\pm$ 1506.58 & 4455.29 $\pm$ 1929.40 \\
{M-klUCB} & 23635.16 $\pm$ 271.90 & 18186.66 $\pm$ 448.59 & 12342.55 $\pm$ 508.37 & 7890.39 $\pm$ 676.83 & 5157.24 $\pm$ 1093.28 & 3513.25 $\pm$ 1658.05 \\
{MASTER} & 16481.34 $\pm$ 298.54 & 15466.47 $\pm$ 389.90 & 14764.19 $\pm$ 511.71 & 14121.37 $\pm$ 684.32 & 13555.03 $\pm$ 1013.75 & 13289.08 $\pm$ 1501.69 \\
{GLR-klUCB Gauss} & 13915.86 $\pm$ 286.75 & 10005.02 $\pm$ 351.70 & 6633.46 $\pm$ 407.41 & 3987.30 $\pm$ 470.06 & 2177.19 $\pm$ 529.03 & 1073.65 $\pm$ 519.08 \\
{GLR-klUCB Bern} & 13334.08 $\pm$ 273.42 & 9489.89 $\pm$ 322.31 & 6224.82 $\pm$ 373.51 & 3682.48 $\pm$ 438.88 & 1982.95$\pm$ 576.81 & 975.58 $\pm$ 399.63 \\
{DAB:B-GSR+klUCB} & 13408.14 $\pm$ 278.97 & 9887.04 $\pm$ 344.98 & 6806.19 $\pm$ 431.44 & 4207.55 $\pm$ 547.14 & 2318.38 $\pm$ 648.37 & 1153.98 $\pm$ 717.81 \\
{DAB:B-GSR+MOSS} & 14406.85 $\pm$ 307.12 & 10601.24 $\pm$ 394.05 & 7091.65 $\pm$ 457.77 & 4197.40 $\pm$ 542.00 & 2267.44 $\pm$ 651.33 & 1114.43 $\pm$ 708.70 \\
{DAB:B-GSR+UCB} & 12877.56 $\pm$ 192.78 & 10634.64 $\pm$ 225.79 & 8367.51 $\pm$ 269.72 & 5940.87 $\pm$ 342.93 & 3709.38 $\pm$ 428.07 & 2048.91 $\pm$ 436.55 \\
{DAB:G-GSR+klUCB} & 13979.56 $\pm$ 319.87 & 10387.42 $\pm$ 376.16 & 7177.38 $\pm$ 458.10 & 4497.57 $\pm$ 570.29 & 2543.21 $\pm$ 709.72 & 1270.88 $\pm$ 825.29 \\
{DAB:G-GSR+MOSS} & 14869.79 $\pm$ 338.75 & 10994.70 $\pm$ 402.17 & 7405.52 $\pm$ 484.74 & 4484.91 $\pm$ 578.97 & 2435.58 $\pm$ 667.76 & 1196.88 $\pm$ 695.87 \\
{DAB:G-GSR+UCB} & 12892.65 $\pm$ 231.79 & 10192.56 $\pm$ 254.93 & 7836.69 $\pm$ 306.90 & 5613.96 $\pm$ 367.54 & 3610.45 $\pm$ 430.64 & 2068.68 $\pm$ 464.58 \\
{DAB:B-GLR+klUCB} & 13246.57 $\pm$ 270.41 & \textbf{9428.66 $\pm$ 336.55} & \textbf{6191.07 $\pm$ 397.37} & \textbf{3670.76 $\pm$ 483.81} & \textbf{1980.95 $\pm$ 478.34} & \textbf{972.39 $\pm$ 369.76} \\
{DAB:B-GLR+MOSS} & 13885.32 $\pm$ 261.77 & 10186.61 $\pm$ 322.99 & 6749.52 $\pm$ 389.74 & 3951.04 $\pm$ 474.68 & 2067.48 $\pm$ 489.53 & 1007.97 $\pm$ 551.67 \\
{DAB:B-GLR+UCB} & 13316.54 $\pm$ 175.51 & 11276.16 $\pm$ 208.69 & 8866.54 $\pm$ 255.54 & 6165.72 $\pm$ 341.18 & 3744.94 $\pm$ 410.86 & 2029.76 $\pm$ 418.76 \\
{DAB:G-GLR+klUCB} & 13824.25 $\pm$ 301.19 & 9938.51 $\pm$ 352.11 & 6591.21 $\pm$ 427.18 & 3979.95 $\pm$ 531.41 & 2173.97 $\pm$ 616.32 & 1059.89 $\pm$ 635.49 \\
{DAB:G-GLR+MOSS} & 14287.89 $\pm$ 291.64 & 10488.16 $\pm$ 359.74 & 6989.79 $\pm$ 443.02 & 4161.08 $\pm$ 517.49 & 2218.64 $\pm$ 580.43 & 1080.51 $\pm$ 584.17 \\
{DAB:G-GLR+UCB} & \textbf{12807.55 $\pm$ 184.81} & 10476.87 $\pm$ 226.95 & 8198.70 $\pm$ 281.85 & 5853.98 $\pm$ 343.80 & 3687.73 $\pm$ 417.37 & 2045.31 $\pm$ 427.10 \\
\hline
\end{tabular}
\end{center}
\end{table}

\paragraph{Regret Performance} From the results in Table~\ref{table:final_regret} and Figure~\ref{fig:reg_results}, our modularized DAB procedures appear to compare favorably with the state-of-art GLR-klUCB approach of \cite{besson2022efficient} in terms of regret minimization, and surpass the latter's performance in many instances. The efficacy of our modularized approach is closely tied to the chosen bandit algorithm. This observation aligns with our expectations and it underscores the flexibility in our design where the choice of the bandit algorithm can be tailored to the specific application and available computational resources. For example, employing UCB or MOSS as the bandit algorithm can offer a speedup compared to employing klUCB,  while at times it outperforms in terms of the regret (see, e.g., UCB for $\xi = 0.3$ in Table \ref{table:final_regret}). 

For the DAB procedure, the variance of the regret depends on  the selection of the stationary bandit algorithm, with UCB exhibiting the lowest value. Even though GLR-klUCB consistently delivers robust results in terms of variance, the range of the mean regret for the adaptive DAB procedure with Bernoulli GLR test and klUCB is lower compared to GLR-klUCB.

A noteworthy observation is that UCB without change detection (row 1) consistently outperforms klUCB without change detection (row 2). While this may seem counterintuitive, given that klUCB performs better in stationary settings, it has been shown that UCB pulls suboptimal arms more frequently as the horizon increases \cite{garivier2011kl}. In a non-stationary setting, this higher rate of exploration enables better adaptation to changes, which also explains why the DAB procedure achieves the lowest regret with UCB under high non-stationarity $(\xi=0.3)$. 

A final key result concerns the performance of the DAB procedures in scenarios with high non-stationarity. Specifically, when $\xi \leq 0.5$, condition \eqref{eq:cond1} is likely to fail often due to the high average number of change-points within the horizon. Remarkably, even under such challenging conditions, the DAB procedures demonstrate excellent performance. To summarize, it is evident that our DAB procedure exhibits strong performance across varying rates of non-stationarity. While GLR-klUCB provides consistently reliable results across settings, our modularized approach offers a flexible, efficient, and equivalently optimal alternative when the degree of the non-stationarity is unknown. In what follows, we take a deeper dive into the detection performance of our DAB procedures.

\paragraph{Detection Performance} 

To assess the effectiveness of our method in detecting changes, we evaluate three aspects: the accuracy of identifying change-points, the detection delay, and the cascading errors that may occur when a change-point is missed in terms of additional missed changes.

\begin{table}[h]
\caption{Average Numbers of Change-Points (CP), Detections (D), True Detections (TD) at $\xi\in\{0.3,0.8\}$ and Average Delay (time steps, averaged over $\xi\in\{0.3,0.4,0.5,0.6,0.7,0.8\}$); $T=100000$, 2000 Runs.}
\label{table:detections_td_with_avg_delay}
\footnotesize
\setlength{\tabcolsep}{5pt}
\centering
\begin{tabular}{l|cc|cc|c}
\hline
\multirow{3}{*}{\textbf{Algorithm}} & \multicolumn{2}{c|}{\textbf{$\xi = 0.3$}} & \multicolumn{2}{c|}{\textbf{$\xi = 0.8$}} & \textbf{Avg.\ Delay} \\
& \multicolumn{2}{c|}{\textbf{CP}: 3163.6} & \multicolumn{2}{c|}{\textbf{CP}: 9.6} & \textbf{(over $\xi$)} \\
\cline{2-5}
& \textbf{D} & \textbf{TD} & \textbf{D} & \textbf{TD} & \\
\hline
DAB:B-GSR+klUCB  & 105.8 & 105.8 &  8.0 &  7.3 & 25.3 \\
DAB:B-GSR+MOSS   & 117.4 & 117.4 &  8.0 &  7.3 & 24.8 \\
DAB:B-GSR+UCB    &  99.0 &  99.0 &  9.0 &  8.2 & 25.5 \\
DAB:G-GSR+klUCB  &  76.4 &  76.4 &  7.2 &  6.6 & 24.4 \\
DAB:G-GSR+MOSS   &  90.7 &  90.7 &  7.1 &  6.6 & 24.8 \\
DAB:G-GSR+UCB    &  59.7 &  59.7 &  8.3 &  7.6 & 24.7 \\
DAB:B-GLR+klUCB  & 175.7 & 175.7 &  8.5 &  7.7 & 24.5 \\
DAB:B-GLR+MOSS   & 180.3 & 180.3 &  8.6 &  7.9 & 24.8 \\
DAB:B-GLR+UCB    & 180.1 & 180.1 &  9.3 &  8.5 & 25.2 \\
DAB:G-GLR+klUCB  & 127.9 & 127.9 &  7.7 &  7.1 & 24.5 \\
DAB:G-GLR+MOSS   & 131.2 & 131.2 &  7.7 &  7.1 & 25.1 \\
DAB:G-GLR+UCB    & 106.3 & 106.3 &  8.9 &  8.1 & 24.9 \\
\hline
\end{tabular}
\end{table}

We first report the average number of detections and true detections for the two extreme cases of non-stationarity ($\xi=0.3$ and $\xi=0.8$) in Table~\ref{table:detections_td_with_avg_delay}. When $\xi \leq 0.5$, the level of non-stationarity is high enough that the average spacing between change-points falls below 320 time steps, making condition~\eqref{eq:cond1} unlikely to hold. Under the most extreme setting ($\xi=0.3$), the average number of change-points is $3163.6$, and our approach detects only $2-5.7\%$ of them. This result is expected: with an average spacing of just 32 time steps, there are typically too few samples to enable reliable detection. Nevertheless, as shown in Figure~\ref{fig:reg_results} and Table~\ref{table:final_regret}, our method still achieves strong overall performance despite the low detection rate.

In contrast, in the mildest case of non-stationarity ($\xi=0.8$), the average number of change-points drops to $9.6$, and our approach achieves up to $88.5\%$ detection accuracy. This improvement stems from the much larger spacing between change-points (approximately 10,200 time steps), which allows sufficient samples for reliable detection. Interestingly, however, a higher detection rate does not necessarily correspond to the best overall performance (cf. B-GLR+UCB vs.\ B-GLR+klUCB).

Finally, we examine the detection delay, measured only from successful detections. Averaged across runs and all six values of $\xi$, the delay is consistently around 25 time steps, with small variance across $\xi$. Moreover, the delay is comparable across all algorithms, highlighting the robustness of our detection mechanism.

\begin{table}[h]
\caption{Missed Change Points: Count Until Next Detection, $T=100000$, 2000 Runs.}
\label{table:missed_count_and_distance_until_next}
\scriptsize
\setlength{\tabcolsep}{3pt}
\begin{center}
\begin{tabular}{l|cccccc}
\hline
\textbf{Algorithm} & \textbf{$\xi = 0.3$} & \textbf{$\xi = 0.4$} & \textbf{$\xi = 0.5$} & \textbf{$\xi = 0.6$} & \textbf{$\xi = 0.7$} & \textbf{$\xi = 0.8$} \\
 & \textbf{CP dist: 32} & \textbf{CP dist: 100} & \textbf{CP dist: 316} & \textbf{CP dist: 1004} & \textbf{CP dist: 3177} & \textbf{CP dist: 10197} \\
\hline
DAB:B-GSR+klUCB  & 29.8 &  8.1 & 3.8 & 2.6 & 2.2 & 2.0 \\
DAB:B-GSR+MOSS   & 26.8 &  7.7 & 3.8 & 2.6 & 2.2 & 2.1 \\
DAB:B-GSR+UCB    & 32.0 &  7.5 & 3.4 & 2.3 & 2.0 & 2.0 \\
DAB:G-GSR+klUCB  & 41.6 & 10.8 & 4.6 & 2.9 & 2.3 & 2.1 \\
DAB:G-GSR+MOSS   & 34.9 & 10.0 & 4.5 & 2.9 & 2.3 & 2.1 \\
DAB:G-GSR+UCB    & 53.6 & 11.3 & 4.4 & 2.6 & 2.1 & 2.0 \\
DAB:B-GLR+klUCB  & 17.7 &  5.9 & 3.3 & 2.4 & 2.1 & 2.0 \\
DAB:B-GLR+MOSS   & 17.2 &  5.7 & 3.2 & 2.4 & 2.1 & 2.0 \\
DAB:B-GLR+UCB    & 17.2 &  5.1 & 2.8 & 2.1 & 2.0 & 2.0 \\
DAB:G-GLR+klUCB  & 24.5 &  7.8 & 3.9 & 2.7 & 2.2 & 2.1 \\
DAB:G-GLR+MOSS   & 23.9 &  7.6 & 3.8 & 2.7 & 2.2 & 2.1 \\
DAB:G-GLR+UCB    & 29.7 &  7.6 & 3.4 & 2.3 & 2.0 & 2.0 \\
\hline
\end{tabular}
\end{center}
\end{table}

We now study the effect of a missed change-point on subsequent time steps. In our theoretical analysis, we conservatively assumed that if a change-point is missed (with probability $\delta_D$), the regret would grow linearly for the remainder of the horizon, since no further detections would be made. This assumption simplifies the analysis but ignores the possibility that the algorithm may recover and detect later changes. To examine this possible ``cascade'' empirically, we measure the number of consecutive change-points missed on average before the next detection.

As shown in Table~\ref{table:missed_count_and_distance_until_next}, in regimes where condition~\eqref{eq:cond1} is likely violated ($\xi \leq 0.5$), the number of missed change-points can be higher, especially at $\xi=0.3$, where the average spacing is just 32 time steps. In contrast, when $\xi > 0.5$, the number of missed change-points stabilizes close to 2 across all algorithms. This finding indicates that the conservative linear-regret from a missed detection assumption overestimates the practical impact said misses---in practice, the detection mechanism is able to recover quickly, preventing a full cascade of missed change-points.

\begin{remark} In Section \ref{sec:DAB_cor}, we showed that the proposed DAB procedures are order-optimal under condition \eqref{eq:cond1}. However, for our geometric change-point model, condition \eqref{eq:cond1} may not always hold. While condition \eqref{eq:cond1} is required to facilitate the theoretical analysis, the failure of this condition does not appear to adversely affect the regret, as we see in the simulations. This is because in the theoretical analysis, missing a single change-point causes the derived regret bound to be linear, whereas in reality, the change-detector is able to quickly recover from this error and keep the regret down by restarting the bandit algorithm after subsequent changes.
\end{remark}

\paragraph{On the Necessity of Forced Exploration}
\label{para:force_exp}
As a final study, we investigate the role of \emph{forced exploration} in DAB procedures. 
From the results presented earlier, setting $\alpha_0 = 0.05$ consistently outperformed $\alpha_0 = 0.1$, suggesting that smaller (and possibly zero) exploration rates may be preferable in practice---a trend also observed in prior work~\cite{besson2022efficient}. 
To further examine this hypothesis, we compare the best-performing DAB variant, (DAB:B-GLR+klUCB), under two settings---$\alpha_0 = 0.05$ and $\alpha_0 = 0$---for $T = 10^7$ and $\xi = 0.7$, using the same evaluation metrics as the experiments before.

\begin{table}[h]
\caption{Combined results for $\xi = 0.7$, $T = 10{,}000{,}000$, averaged over 20 runs.}
\label{table:combined_results_xi07}
\footnotesize
\setlength{\tabcolsep}{5pt}
\begin{center}
\begin{tabular}{l|ccccc}
\hline
\multicolumn{6}{c}{\textbf{CP:} 126.2 \quad \textbf{CP dist:} 79215} \\
\hline
\textbf{Algorithm} & \textbf{Final Regret} & \textbf{Detections} & \textbf{True Det.} & \textbf{Missed Cnt.} & \textbf{Avg. Delay} \\
\hline
DAB:B-GLR+klUCB ($\alpha_0 = 0.05$) & 20{,}728.81 $\pm$ 3{,}531.26 & 121.8 & 121.8 & 1.0 & 35.2 \\
DAB:B-GLR+klUCB ($\alpha_0 = 0$)    & 22{,}259.86 $\pm$ 6{,}532.66 & 120.9 & 120.9 & 1.0 & 35.2 \\
\hline
\end{tabular}
\end{center}
\end{table}

Table~\ref{table:combined_results_xi07} highlights two key observations. 
First, the detection performance of DAB improves markedly for the given non-stationarity level $\xi = 0.7$. 
Detection accuracy rises to 96.55\% at $T = 10^7$ (from 70.15\% at $T = 10^5$), while the false-alarm rate drops to 0\% (from 2.3\%). 
Moreover, the average number of missed change-points after an unsuccessful detection decreases from 2 to 1. These results confirm our theoretical analysis that larger horizons $T$ facilitate more reliable change detection. 
As expected, the average detection delay increases with $T$, in line with our asymptotic delay bounds. Second, comparing the two exploration levels reveals that a nonzero forced exploration ($\alpha_0 > 0$) remains beneficial in the large-horizon regime. 
While both settings achieve nearly identical detection statistics and delays, the version with $\alpha_0 = 0.05$ attains \emph{lower dynamic regret} and slightly more detections. 
This confirms that forced exploration mitigates potential detection failures that can cascade across change-points. Overall, these findings suggest that although forced exploration may be unnecessary for moderate horizons its presence is advantageous for large horizons, providing a consistent safeguard that enhances the robustness and long-term performance of DAB procedures.

\section{Conclusions}
\label{sec:sum}
In this paper, we studied DAB procedures, a collection of procedures for PS-MABs that meaningfully marry  efficient stationary MAB algorithms with efficient change detectors. We first established the lower bounds on the performance of PS-MABs, introducing a novel bound for the instance-dependent case and sharpening prior art for the instance-independent case. Then, we introduced our modular DAB procedure, and identified the requirements needed for the change detector and the bandit algorithm, showcasing existing detectors and bandit algorithms that satisfy these conditions. While a good learning algorithm for the stationary variant should quickly learn an optimal arm and favor pulling said arm over time, detection of changes in the reward structures in other arms requires pulling sub-optimal arms often enough (forced exploration) for change detectors to function effectively. The key contribution of this work is a modular approach to the regret analysis of DAB procedures. Using this modular approach, along with an appropriately designed forced exploration policy, we derived order optimal regret bounds (up to logarithmic factors) for an array of DAB procedures. Our experiments demonstrated the efficacy and robustness of these DAB procedures for PS-MABs both in terms of regret and detection accuracy.

Our contributions are centered on PS-MABs, where the changes occur abruptly and with sufficiently low frequency. It is clearly of interest to extend these results to more general non-stationary MABs with possibly slowly-changing environments. The work in \cite{wei2021non} provided a general procedure for a general class of non-stationary reinforcement learning problems that can tackle both piecewise stationary environments and slowly-changing ones. However, it was recently shown that the regret bound given in \cite{wei2021non}, while being order optimal (disregarding polylog factors),  is loose for a critical range of finite horizons \cite{gerogiannis2024ipfnsrlopt}. Furthermore, it was shown in \cite{gerogiannis2024ipfnsrlopt} that 
for the (special) case of PS-MABs, the procedure given in \cite{wei2021non} performs poorly compared to DAB procedures. 
Therefore, developing efficient procedures for slowly changing non-stationary MAB environments remains largely an open problem. 


\bibliographystyle{IEEEtran}
\bibliography{TIT.bib}

\begin{thebibliography}{10}
\providecommand{\url}[1]{#1}
\csname url@samestyle\endcsname
\providecommand{\newblock}{\relax}
\providecommand{\bibinfo}[2]{#2}
\providecommand{\BIBentrySTDinterwordspacing}{\spaceskip=0pt\relax}
\providecommand{\BIBentryALTinterwordstretchfactor}{4}
\providecommand{\BIBentryALTinterwordspacing}{\spaceskip=\fontdimen2\font plus
\BIBentryALTinterwordstretchfactor\fontdimen3\font minus \fontdimen4\font\relax}
\providecommand{\BIBforeignlanguage}[2]{{%
\expandafter\ifx\csname l@#1\endcsname\relax
\typeout{** WARNING: IEEEtran.bst: No hyphenation pattern has been}%
\typeout{** loaded for the language `#1'. Using the pattern for}%
\typeout{** the default language instead.}%
\else
\language=\csname l@#1\endcsname
\fi
#2}}
\providecommand{\BIBdecl}{\relax}
\BIBdecl

\bibitem{lattimore2020bandit}
T.~Lattimore and C.~Szepesv{\'a}ri, \emph{Bandit algorithms}.\hskip 1em plus 0.5em minus 0.4em\relax Cambridge University Press, 2020.

\bibitem{slivkins2019introduction}
A.~Slivkins \emph{et~al.}, ``Introduction to multi-armed bandits,'' \emph{Foundations and Trends{\textregistered} in Machine Learning}, vol.~12, no. 1-2, pp. 1--286, 2019.

\bibitem{li2010recomm}
\BIBentryALTinterwordspacing
L.~Li, W.~Chu, J.~Langford, and R.~E. Schapire, ``A contextual-bandit approach to personalized news article recommendation,'' in \emph{Proceedings of the 19th International Conference on World Wide Web}, ser. WWW '10.\hskip 1em plus 0.5em minus 0.4em\relax New York, NY, USA: Association for Computing Machinery, 2010, p. 661–670. [Online]. Available: \url{https://doi.org/10.1145/1772690.1772758}
\BIBentrySTDinterwordspacing

\bibitem{lefortier2014recomm}
\BIBentryALTinterwordspacing
D.~Lefortier, P.~Serdyukov, and M.~de~Rijke, ``Online exploration for detecting shifts in fresh intent,'' in \emph{Proceedings of the 23rd ACM International Conference on Conference on Information and Knowledge Management}, ser. CIKM '14.\hskip 1em plus 0.5em minus 0.4em\relax New York, NY, USA: Association for Computing Machinery, 2014, p. 589–598. [Online]. Available: \url{https://doi.org/10.1145/2661829.2661947}
\BIBentrySTDinterwordspacing

\bibitem{chapelle2011advr}
\BIBentryALTinterwordspacing
O.~Chapelle and L.~Li, ``An empirical evaluation of thompson sampling,'' in \emph{Advances in Neural Information Processing Systems}, J.~Shawe-Taylor, R.~Zemel, P.~Bartlett, F.~Pereira, and K.~Weinberger, Eds., vol.~24.\hskip 1em plus 0.5em minus 0.4em\relax Curran Associates, Inc., 2011. [Online]. Available: \url{https://proceedings.neurips.cc/paper_files/paper/2011/file/e53a0a2978c28872a4505bdb51db06dc-Paper.pdf}
\BIBentrySTDinterwordspacing

\bibitem{sertan2012advr}
\BIBentryALTinterwordspacing
G.~Sertan, M.~Jérémie, P.~Philippe, and N.~Olivier, ``Managing advertising campaigns—an approximate planning approach,'' \emph{Frontiers of Computer Science}, vol.~6, no.~2, p. 209, 2012. [Online]. Available: \url{https://journal.hep.com.cn/fcs/EN/abstract/article_3688.shtml}
\BIBentrySTDinterwordspacing

\bibitem{schwartz2017onlineadvr}
E.~M. Schwartz, E.~T. Bradlow, and P.~S. Fader, ``Customer acquisition via display advertising using multi-armed bandit experiments,'' \emph{Marketing Science}, vol.~36, no.~4, pp. 500--522, 2017.

\bibitem{tajik2024dynamicpr}
\BIBentryALTinterwordspacing
M.~Tajik, B.~M. Tosarkani, A.~Makui, and R.~Ghousi, ``A novel two-stage dynamic pricing model for logistics planning using an exploration–exploitation framework: A multi-armed bandit problem,'' \emph{Expert Systems with Applications}, vol. 246, p. 123060, 2024. [Online]. Available: \url{https://www.sciencedirect.com/science/article/pii/S0957417423035625}
\BIBentrySTDinterwordspacing

\bibitem{flajolet2017bidding}
\BIBentryALTinterwordspacing
A.~Flajolet and P.~Jaillet, ``Real-time bidding with side information,'' in \emph{Advances in Neural Information Processing Systems}, I.~Guyon, U.~V. Luxburg, S.~Bengio, H.~Wallach, R.~Fergus, S.~Vishwanathan, and R.~Garnett, Eds., vol.~30.\hskip 1em plus 0.5em minus 0.4em\relax Curran Associates, Inc., 2017. [Online]. Available: \url{https://proceedings.neurips.cc/paper_files/paper/2017/file/0bed45bd5774ffddc95ffe500024f628-Paper.pdf}
\BIBentrySTDinterwordspacing

\bibitem{cai2017nsbid}
\BIBentryALTinterwordspacing
H.~Cai, K.~Ren, W.~Zhang, K.~Malialis, J.~Wang, Y.~Yu, and D.~Guo, ``Real-time bidding by reinforcement learning in display advertising,'' in \emph{Proceedings of the Tenth ACM International Conference on Web Search and Data Mining}, ser. WSDM '17.\hskip 1em plus 0.5em minus 0.4em\relax New York, NY, USA: Association for Computing Machinery, 2017, p. 661–670. [Online]. Available: \url{https://doi.org/10.1145/3018661.3018702}
\BIBentrySTDinterwordspacing

\bibitem{lu2019nsbid2}
\BIBentryALTinterwordspacing
J.~Lu, C.~Yang, X.~Gao, L.~Wang, C.~Li, and G.~Chen, ``Reinforcement learning with sequential information clustering in real-time bidding,'' in \emph{Proceedings of the 28th ACM International Conference on Information and Knowledge Management}, ser. CIKM '19.\hskip 1em plus 0.5em minus 0.4em\relax New York, NY, USA: Association for Computing Machinery, 2019, p. 1633–1641. [Online]. Available: \url{https://doi.org/10.1145/3357384.3358027}
\BIBentrySTDinterwordspacing

\bibitem{chen2020trafficns}
\BIBentryALTinterwordspacing
C.~Chen, H.~Wei, N.~Xu, G.~Zheng, M.~Yang, Y.~Xiong, K.~Xu, and Z.~Li, ``Toward a thousand lights: Decentralized deep reinforcement learning for large-scale traffic signal control,'' \emph{Proceedings of the AAAI Conference on Artificial Intelligence}, vol.~34, no.~04, pp. 3414--3421, Apr. 2020. [Online]. Available: \url{https://ojs.aaai.org/index.php/AAAI/article/view/5744}
\BIBentrySTDinterwordspacing

\bibitem{kocsis2006discounted}
L.~Kocsis and C.~Szepesv{\'a}ri, ``Discounted ucb,'' in \emph{2nd PASCAL Challenges Workshop}, vol.~2, 2006, pp. 51--134.

\bibitem{auer2002using}
P.~Auer, ``Using confidence bounds for exploitation-exploration trade-offs,'' \emph{Journal of Machine Learning Research}, vol.~3, no. Nov, pp. 397--422, 2002.

\bibitem{seznec2020single}
J.~Seznec, P.~Menard, A.~Lazaric, and M.~Valko, ``A single algorithm for both restless and rested rotting bandits,'' in \emph{International Conference on Artificial Intelligence and Statistics}.\hskip 1em plus 0.5em minus 0.4em\relax PMLR, 2020, pp. 3784--3794.

\bibitem{auer2002finite}
P.~Auer, N.~Cesa-Bianchi, and P.~Fischer, ``Finite-time analysis of the multiarmed bandit problem,'' \emph{Machine learning}, vol.~47, pp. 235--256, 2002.

\bibitem{audibert2009minimax}
J.-Y. Audibert, S.~Bubeck \emph{et~al.}, ``Minimax policies for adversarial and stochastic bandits.'' in \emph{COLT}, vol.~7, 2009, pp. 1--122.

\bibitem{tartakovsky-veeravalli-2005}
A.~G. Tartakovsky and V.~V. Veeravalli, ``General asymptotic bayesian theory of quickest change detection,'' \emph{Theory of Probability \& Its Applications}, vol.~49, no.~3, pp. 458--497, 2005.

\bibitem{Poor_Hadjiliadis_2008}
H.~V. Poor and O.~Hadjiliadis, \emph{Quickest Detection}.\hskip 1em plus 0.5em minus 0.4em\relax Cambridge University Press, 2008.

\bibitem{vvv_qcd_overview}
V.~V. Veeravalli and T.~Banerjee, ``Quickest change detection,'' in \emph{Academic press library in signal processing: Array and statistical signal processing}.\hskip 1em plus 0.5em minus 0.4em\relax Cambridge, MA: Academic Press, 2013.

\bibitem{garivier2011upper}
A.~Garivier and E.~Moulines, ``On upper-confidence bound policies for switching bandit problems,'' in \emph{International Conference on Algorithmic Learning Theory}.\hskip 1em plus 0.5em minus 0.4em\relax Springer, 2011, pp. 174--188.

\bibitem{besbes_mab2014}
\BIBentryALTinterwordspacing
O.~Besbes, Y.~Gur, and A.~Zeevi, ``Stochastic multi-armed-bandit problem with non-stationary rewards,'' in \emph{Advances in Neural Information Processing Systems}, Z.~Ghahramani, M.~Welling, C.~Cortes, N.~Lawrence, and K.~Weinberger, Eds., vol.~27.\hskip 1em plus 0.5em minus 0.4em\relax Curran Associates, Inc., 2014. [Online]. Available: \url{https://proceedings.neurips.cc/paper_files/paper/2014/file/903ce9225fca3e988c2af215d4e544d3-Paper.pdf}
\BIBentrySTDinterwordspacing

\bibitem{abbasi2023new}
Y.~Abbasi-Yadkori, A.~Gy{\"o}rgy, and N.~Lazi{\'c}, ``A new look at dynamic regret for non-stationary stochastic bandits,'' \emph{Journal of Machine Learning Research}, vol.~24, no. 288, pp. 1--37, 2023.

\bibitem{liu2018change}
\BIBentryALTinterwordspacing
F.~Liu, J.~Lee, and N.~Shroff, ``A change-detection based framework for piecewise-stationary multi-armed bandit problem,'' \emph{Proceedings of the AAAI Conference on Artificial Intelligence}, vol.~32, no.~1, Apr. 2018. [Online]. Available: \url{https://ojs.aaai.org/index.php/AAAI/article/view/11746}
\BIBentrySTDinterwordspacing

\bibitem{cao2019nearly}
Y.~Cao, Z.~Wen, B.~Kveton, and Y.~Xie, ``Nearly optimal adaptive procedure with change detection for piecewise-stationary bandit,'' in \emph{The 22nd International Conference on Artificial Intelligence and Statistics}.\hskip 1em plus 0.5em minus 0.4em\relax PMLR, 2019, pp. 418--427.

\bibitem{besson2022efficient}
L.~Besson, E.~Kaufmann, O.-A. Maillard, and J.~Seznec, ``Efficient change-point detection for tackling piecewise-stationary bandits,'' \emph{The Journal of Machine Learning Research}, vol.~23, no.~1, pp. 3337--3376, 2022.

\bibitem{hartland2006multi}
\BIBentryALTinterwordspacing
C.~Hartland, S.~Gelly, N.~Baskiotis, O.~Teytaud, and M.~Sebag, ``{Multi-armed Bandit, Dynamic Environments and Meta-Bandits},'' Nov. 2006, working paper or preprint. [Online]. Available: \url{https://hal.science/hal-00113668}
\BIBentrySTDinterwordspacing

\bibitem{page1954continuous}
E.~S. Page, ``Continuous inspection schemes,'' \emph{Biometrika}, vol.~41, no. 1/2, pp. 100--115, 1954.

\bibitem{lorden1971procedures}
G.~Lorden, ``Procedures for reacting to a change in distribution,'' \emph{The annals of mathematical statistics}, pp. 1897--1908, 1971.

\bibitem{auer2019achieving}
P.~Auer, Y.~Chen, P.~Gajane, C.-W. Lee, H.~Luo, R.~Ortner, and C.-Y. Wei, ``Achieving optimal dynamic regret for non-stationary bandits without prior information,'' in \emph{Conference on Learning Theory}.\hskip 1em plus 0.5em minus 0.4em\relax PMLR, 2019, pp. 159--163.

\bibitem{huang2025sequential}
\BIBentryALTinterwordspacing
Y.-H. Huang and V.~V. Veeravalli, ``Sequential change detection for learning in piecewise stationary bandit environments,'' in \emph{2025 IEEE International Symposium on Information Theory (ISIT)}.\hskip 1em plus 0.5em minus 0.4em\relax IEEE, 2025. [Online]. Available: \url{https://arxiv.org/abs/2501.10974}
\BIBentrySTDinterwordspacing

\bibitem{wei2021non}
C.-Y. Wei and H.~Luo, ``Non-stationary reinforcement learning without prior knowledge: An optimal black-box approach,'' in \emph{Conference on learning theory}.\hskip 1em plus 0.5em minus 0.4em\relax PMLR, 2021, pp. 4300--4354.

\bibitem{gerogiannis2024ipfnsrlopt}
\BIBentryALTinterwordspacing
A.~Gerogiannis, Y.-H. Huang, and V.~Veeravalli, ``Is prior-free black-box non-stationary reinforcement learning feasible?'' in \emph{The 28th International Conference on Artificial Intelligence and Statistics}, 2025. [Online]. Available: \url{https://openreview.net/forum?id=OqOf8mtvVW}
\BIBentrySTDinterwordspacing

\bibitem{auer2019adswitch}
P.~Auer, P.~Gajane, and R.~Ortner, ``Adaptively tracking the best bandit arm with an unknown number of distribution changes,'' in \emph{Proceedings of the Thirty-Second Conference on Learning Theory}, ser. Proceedings of Machine Learning Research, A.~Beygelzimer and D.~Hsu, Eds., vol.~99.\hskip 1em plus 0.5em minus 0.4em\relax PMLR, 25--28 Jun 2019, pp. 138--158.

\bibitem{besson2018doublingtrickscantmultiarmed}
L.~Besson and E.~Kaufmann, ``What doubling tricks can and can't do for multi-armed bandits,'' 2018.

\bibitem{moustakides1986optimal}
G.~V. Moustakides, ``Optimal stopping times for detecting changes in distributions,'' \emph{the Annals of Statistics}, vol.~14, no.~4, pp. 1379--1387, 1986.

\bibitem{lai1998information}
T.~L. Lai, ``Information bounds and quick detection of parameter changes in stochastic systems,'' \emph{IEEE Transactions on Information theory}, vol.~44, no.~7, pp. 2917--2929, 1998.

\bibitem{cappe2013kullback}
O.~Capp{\'e}, A.~Garivier, O.-A. Maillard, R.~Munos, and G.~Stoltz, ``Kullback-leibler upper confidence bounds for optimal sequential allocation,'' \emph{The Annals of Statistics}, pp. 1516--1541, 2013.

\bibitem{garivier2011kl}
A.~Garivier and O.~Capp{\'e}, ``The kl-ucb algorithm for bounded stochastic bandits and beyond,'' in \emph{Proceedings of the 24th annual conference on learning theory}.\hskip 1em plus 0.5em minus 0.4em\relax JMLR Workshop and Conference Proceedings, 2011, pp. 359--376.

\bibitem{bretagnolle1979estimation}
J.~Bretagnolle and C.~Huber, ``Estimation des densit{\'e}s: risque minimax,'' \emph{Zeitschrift f{\"u}r Wahrscheinlichkeitstheorie und verwandte Gebiete}, vol.~47, no.~2, pp. 119--137, 1979.

\end{thebibliography}

\appendices

\section{Proof of Theorems \ref{thm:reg_low_inst_dep} and \ref{thm:reg_low_minimax}}
\label{sec:proof_lower_bound}

To prove the lower bounds in Theorems \ref{thm:reg_low_inst_dep} and \ref{thm:reg_low_minimax}, we create $2^{N_{T} + 1}$ PS-MAB instances and show that the average regret of these instances is lower bounded by the right-hand side terms in \eqref{eq:reg_low_inst_dep} and \eqref{eq:reg_low_minimax}. 
We first prove Theorem \ref{thm:reg_low_inst_dep} and then use an intermediate step in the proof of Theorem \ref{thm:reg_low_inst_dep} to derive
Theorem \ref{thm:reg_low_minimax}.

Fix an arbitrary policy $\pi$. We construct $2^{N_{T} + 1}$ PS-MAB instances $\lbp B_{j}: j \in \lbp 0, 1 \rbp^{N_{T} + 1} \rbp$, each indexed by an $\lp N_{T} + 1 \rp$-dimensional binary vector $j \coloneqq \lb j_{1} \quad j_{2} \quad \cdots \quad j_{N_{T} + 1} \rb$.
Under each instance $B_{j}$, the reward samples are mutually independent and Gaussian with variance $\sigma^{2}$.
Each instance has $A$ arms with at most $N_{T}$ changes over a horizon of length $T$. Recall that  we assume $T$ is divisible by $N_{T} + 1$.
The horizon is partitioned into $N_{T} + 1$ stationary intervals with equal length. 
More specifically, let $\mcal{I}_{k} \coloneqq \lbp \nu_{k-1}, \dots, \nu_{k} - 1 \rbp$ be the $k^{\mrm{th}}$ stationary interval. Then, for each $k \in \lb N_{T} + 1 \rb$,
\begin{equation}\label{eq:stationary-inv}
    \nu_{k} = k\lp\frac{T}{N_{T}+1} \rp + 1.
\end{equation}
%
Let $\Pr_{j,\pi}$ and $\E_{j,\pi}$ denote the probability measure and expectation in bandit instance $B_{j}$ with policy $\pi$ for $j \in \lbp 0, 1 \rbp^{N_{T} + 1}$. Furthermore, let $\mu_{a, k}$ denote the mean reward of arm $a$ during $\mcal{I}_{k}$ and $n_{a, k}$ denote the number of pulls of arm $a$ during $\mcal{I}_{k}$ for $a \in \lb A \rb$ and $k \in \lb N_{T}+1 \rb$. 
For the purpose of mean reward assignment in bandit instances, define $a_{j,k}$ to be the arm expected to be pulled the least (excluding arm $1$) during interval $\mcal{I}_{k}$ in bandit instance $\mcal{B}_{j}$, i.e., 
\begin{equation}\label{eq:ajk}
    a_{j,k} = \argmin_{a \neq 1}\E_{j,\pi} \lb n_{a,k} \rb.
\end{equation}
Starting from the bandit instance with all-zero index vector, we set the mean reward of arm $1$ to be $\Delta$ and those of other arms to be $0$ for all intervals, i.e., $\mu_{1, k} = \Delta$ and $\mu_{a, k} = 0$ for $a \in \lbp 2, \dots, A \rbp$ and $k \in \lb N_{T} + 1 \rb$. 
Then, we proceed to bandit instances with one-hot vectors $B_{j}$. To do so, we first run the policy $\pi$ over the bandit instance indexed by all zeroes. Then, for interval $\mcal{I}_{k}$ with $j_{k} = 0$, we set $\mu_{1, k} = \Delta$ and $\mu_{a, k} = 0$ for $a \in \lbp 2, \dots, A \rbp$, the same as in the bandit with the all-zero vector index. For interval $\mcal{I}_{k}$ with $j_{k} = 1$, we set $\mu_{1, k} = \Delta$, $\mu_{a_{i,k}, k} = 2\Delta$, and $\mu_{a, k} = 0$ for $a \notin \lbp 1, a_{i,k} \rbp$, where $i$ denotes the all-zero vector.
Next, we assign the mean rewards in bandit instances with index vectors with two $1$'s as follows. We begin by running $\pi$ on the bandit instances with one-hot vectors as indices. Then, for interval $\mcal{I}_{k}$ with $j_{k} = 0$, we set the mean rewards as in the bandit indexed by all-zero vector. For interval $\mcal{I}_{k}$ with $j_{k} = 1$, we set $\mu_{1, k} = \Delta$, $\mu_{a_{i,k}, k} = 2\Delta$, and $\mu_{a, k} = 0$ for $a \notin \lbp 1, a_{i,k} \rbp$, where $i$ is the index vector that differs with $j$ only at the $k^{\mrm{th}}$ bit. This completes the construction of the bandit instances indexed by two-hot vectors. Continue this process until the mean rewards are determined in the bandit instance with the all-one vector as its index.
%
%
The intuition behind the construction is as follows. In this mean assignment process, consider two instances that differ only at the $k$-th interval. Among these two, the most neglected suboptimal arm in the bandit instance with $j_k=0$ is made optimal in the instance with $j_k=1$. Consequently, 
if $\pi$ performed well on the $k$-th interval in one of the instances, it will incur significant regret over the same interval in the other instance.
The construction of the bandit instances with $N_{T}=2$ is illustrated in Figure \ref{fig:b_constr}.
Also note that for any two bandit instances with index vectors $j$ and $j'$ sharing the same first $k$ bits, the mean rewards before $\nu_{k}$ are identical, and hence, the regret performance of any causal policy $\pi$ up to $\nu_k$ must also coincide.
%
\begin{figure}[ht]
    \centering
    \includegraphics[width=1\linewidth]{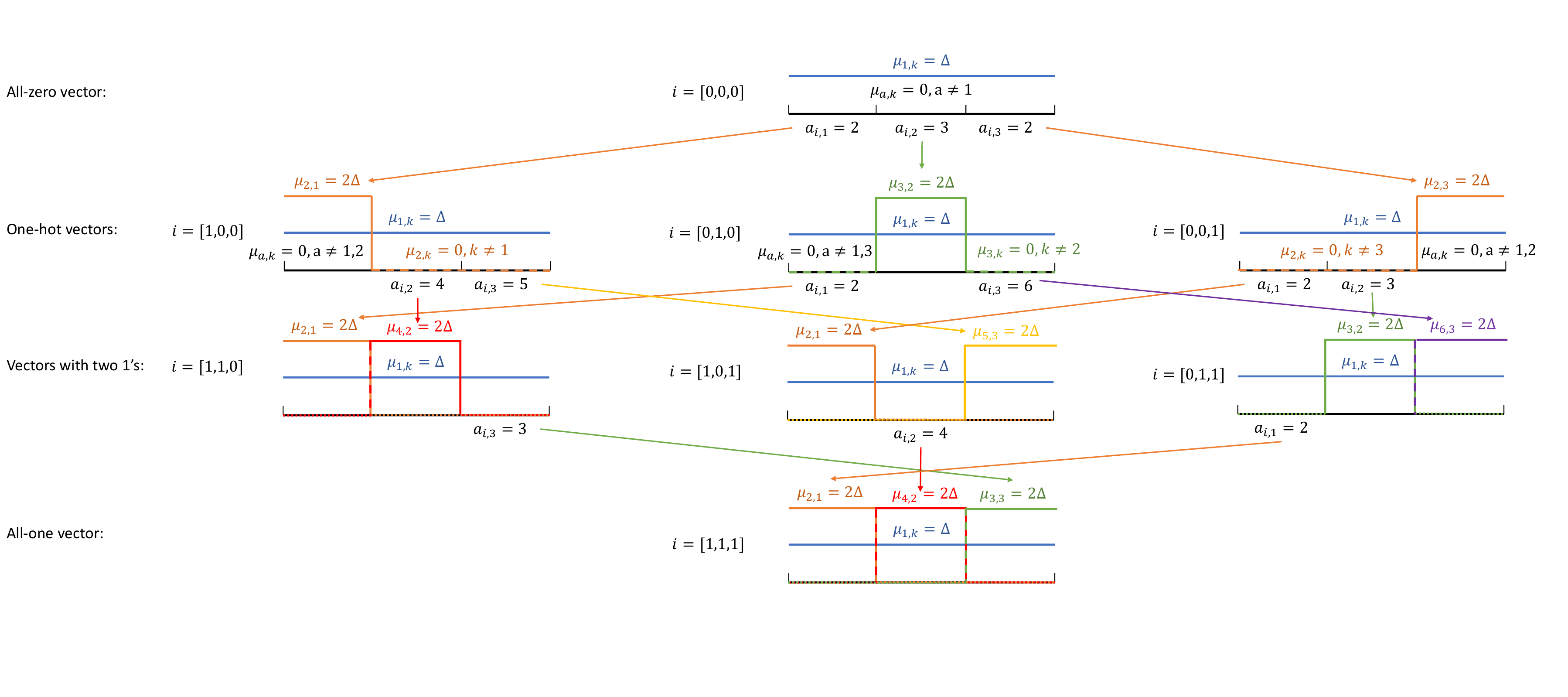}
    \caption{The mean reward assignment process of the bandit instances with $N_{T}=2$.}
    \label{fig:b_constr}
\end{figure}
Now, fix an arbitrary $i \in \lbp 0, 1 \rbp^{N_{T} + 1}$ and $k \in \lb N_{T} + 1 \rb$. We use $R_{j,k}$ to denote the regret (of policy $\pi$) over  interval $\mcal{I}_{k}$ in bandit instance $B_{j}$, and $i$ is the binary vector that differs from $j$ at position $k$. Without loss of generality, we assume that $i_{k} = 0$ and $j_{k} = 1$. Since the optimal arm is arm $1$ and the suboptimality gap for each suboptimal arm is $\Delta$ over interval $\mcal{I}_{k}$ in $B_{i}$, we have
%
\begin{equation}
\begin{split}
    R_{i,k} &= \Delta \E_{i,\pi} \lb \lba \mcal{I}_{k} \rba - n_{1,k} \rb. \label{eq:Rik-1}
\end{split}
\end{equation}
Let $M\coloneqq\lba \mcal{I}_{k} \rba=T/\lp N_{T} + 1 \rp$. Then,
\begin{equation}
\begin{split}
    R_{i,k} &\geq \Delta \E_{i,\pi} \lb \lp M - n_{1,k} \rp \Indc{n_{1,k} \leq \frac{M}{2}} \rb \\
    &\geq \Delta \E_{i,\pi} \lb \frac{M}{2} \Indc{n_{1,k} \leq \frac{M}{2}} \rb \\
    &= \Delta \frac{M}{2}\Pr_{i,\pi} \lp n_{1,k} \leq \frac{M}{2} \rp. \label{eq:Rik}
\end{split}
\end{equation}
Similarly, since the suboptimality gap is $\Delta$ for arm $1$ and those are $2\Delta$ for other suboptimal arms, we have
\begin{equation}
\begin{split}
    R_{j,k} &= \Delta \E_{j,\pi} \lb n_{1,k} \rb + \sum_{a=2, a \neq a_{j,k}}^{A} 2\Delta \E_{j,\pi} \lb n_{a,k} \rb,
\end{split}
\end{equation}
and thus,
\begin{equation}
\begin{split}
    R_{j,k} &\geq \Delta \E_{j,\pi} \lb n_{1,k} \rb \\
    &\geq \Delta \E_{j,\pi} \lb n_{1,k} \Indc{n_{1,k} > \frac{M}{2}} \rb \\
    &\geq \Delta \E_{j,\pi} \lb \frac{M}{2} \Indc{n_{1,k} > \frac{M}{2}} \rb \\
    &= \Delta \frac{M}{2}\Pr_{j,\pi} \lp n_{1,k} > \frac{M}{2} \rp \label{eq:Rjk}.
\end{split}
\end{equation}
Let $i' \coloneqq \lb i_{1} \quad \cdots \quad i_{k} \quad 0 \quad \cdots \quad 0 \rb$ and $j' \coloneqq \lb j_{1} \quad \cdots \quad j_{k} \quad 0 \quad \cdots \quad 0 \rb$. Combining \eqref{eq:Rik} and \eqref{eq:Rjk}, we obtain
\begin{equation}
\begin{split}\label{eq:prob_low_bd}
    R_{i,k} + R_{j,k} &\geq  \frac{\Delta}{2} M  \lb \Pr_{i,\pi} \lp n_{1,k} \leq \frac{\nu_{k} - \nu_{k - 1}}{2} \rp + \Pr_{j,\pi} \lp n_{1,k} > \frac{M}{2} \rp \rb \\
    &\overset{(a)}{=}\frac{\Delta}{2} M \lb \Pr_{i',\pi} \lp n_{1,k} \leq \frac{M}{2} \rp + \Pr_{j',\pi} \lp n_{1,k} > \frac{M}{2} \rp \rb \\
    &\overset{(b)}{\geq} \frac{\Delta}{4} M \exp \lp -\mathrm{kl} \lp \Pr_{i',\pi} || \Pr_{j',\pi} \rp \rp 
\end{split}
\end{equation}
where $\mathrm{kl}$ denotes the KL-divergence. In step $(a)$, we apply change of measure, as $n_{1,k}$ is measurable with respect to reward and action samples from the first $k$ stationary intervals, and $\Pr_{i,\pi}=\Pr_{i',\pi}$ and $\Pr_{j,\pi}=\Pr_{j',\pi}$ during the first $k$ stationary intervals. In step $(b)$, we leverage Bretagnolle-Huber inequality \cite{bretagnolle1979estimation}. In the following inequality, we slightly abuse  notation and use $\mu_{i', t}(a)$ (and $\mu_{j', t}(a)$) to denote the mean reward of arm $a$ at time $t$ in bandit instance $B_{i'}$ (and $B_{j'}$). The KL-divergence between $\Pr_{i',\pi}$ and $\Pr_{j',\pi}$ can be computed as follows:
\begin{align}
    &\mathrm{kl} \lp \Pr_{i',\pi} || \Pr_{j',\pi} \rp \nonumber\\
    &=\int_{\mathbf{x} \in \mathbb{R}^{T}, \mathbf{a} \in \lb A \rb^{T}} \log \lp \frac{\prod_{t = 1}^{T} \pi \lp a_{t} | x_{1}, a_{1}, \dots, x_{t-1}, \dots, a_{t-1} \rp \exp \lp -\frac{1}{2\sigma^{2}}\lp x_{t} - \mu_{i', t}\lp a_{t} \rp \rp^{2}\rp}{\prod_{t = 1}^{T} \pi \lp a_{t} | x_{1}, a_{1}, \dots, x_{t-1}, \dots, a_{t-1} \rp \exp \lp -\frac{1}{2\sigma^{2}}\lp x_{t} - \mu_{j', t}\lp a_{t} \rp \rp^{2}\rp} \rp d\Pr_{i', \pi}\lp \mathbf{x}, \mathbf{a} \rp \nonumber\\
    &=\int_{\mathbf{x} \in \mathbb{R}^{T}, \mathbf{a} \in \lb A \rb^{T}} \frac{1}{2\sigma^{2}} \sum_{t=1}^{T} \lb\lp x_{t} - \mu_{j', t}\lp a_{t} \rp \rp^{2}-\lp x_{t} - \mu_{i', t}\lp a_{t} \rp \rp^{2}\rb d\Pr_{i', \pi}\lp \mathbf{x}, \mathbf{a} \rp \nonumber\\
    &=\frac{1}{2\sigma^{2}} \sum_{t=1}^{T}\int_{\mathbf{x} \in \mathbb{R}^{T}, \mathbf{a} \in \lb A \rb^{T}} \lp \mu_{i', t}\lp a_{t} \rp - \mu_{j', t}\lp a_{t} \rp \rp \lp 2x_{t} - \mu_{i', t}\lp a_{t} \rp - \mu_{j', t}\lp a_{t} \rp \rp d\Pr_{i', \pi}\lp \mathbf{x}, \mathbf{a} \rp \nonumber\\
    &\overset{(a)}{=}-\frac{\Delta}{\sigma^{2}} \sum_{t=\nu_{k-1}}^{\nu_{k}-1}\int_{\mathbf{x} \in \mathbb{R}^{T}, \mathbf{a} \in \lb A \rb^{T}} \lp 2x_{t} - 2\Delta \rp \Indc{a_{t} = a_{i,k}} d\Pr_{i', \pi}\lp \mathbf{x}, \mathbf{a} \rp \nonumber\\
    &=-\frac{\Delta}{\sigma^{2}} \sum_{t=\nu_{k-1}}^{\nu_{k}-1} \E_{i', \pi} \lb \lp 2X_{t} - 2\Delta \rp \Indc{A_{t} = a_{i,k}} \rb \nonumber\\
    &\overset{(b)}{=}-\frac{\Delta}{\sigma^{2}} \sum_{t=\nu_{k-1}}^{\nu_{k}-1} \E_{i, \pi} \lb \lp 2X_{t} - 2\Delta \rp \Indc{A_{t} = a_{i,k}} \rb \nonumber\\
    &=-\frac{\Delta}{\sigma^{2}} \sum_{t=\nu_{k-1}}^{\nu_{k}-1} \E_{i, \pi} \lb \E_{i, \pi} \lb 2X_{t} - 2\Delta | A_{t} = a_{i,k} \rb \Indc{A_{t} = a_{i,k}} \rb \nonumber\\
    &\overset{(c)}{=}\frac{2\Delta^{2}}{\sigma^{2}} \sum_{t=\nu_{k-1}}^{\nu_{k}-1} \E_{i, \pi} \lb \Indc{A_{t} = a_{i,k}} \rb \nonumber\\
    &=\frac{2\Delta^{2}}{\sigma^{2}}  \E_{i, \pi} \lb n_{a_{i,k},k} \rb\label{eq:kl_upp}
\end{align}
%
By the causality of the policy, the reward distributions of $B_{i'}$ and those of $B_{j'}$ are the same except for those of arm $a_{i,k}$ during the $k^{\mrm{th}}$ stationary interval in step $(a)$. 
In step $(b)$, we apply change of measure, as $X_{t}$ and $A_{t}$ in the summation are measurable with respect to reward and action samples from the first $k$ stationary intervals, and $\Pr_{i,\pi}=\Pr_{i',\pi}$ during the first $k$ stationary intervals. 
In step $(c)$, because $i_{k} = 0$, we have $\E_{i, \pi} \lb X_{t} | A_{t} = a_{i,k} \rb = 0$.
Then, by applying \eqref{eq:kl_upp} to \eqref{eq:prob_low_bd}, we get
\begin{equation}\label{eq:intv_reg_low_inst_dep}
\begin{split}
    R_{i,k} + R_{j,k} \geq \frac{\Delta}{4} M \exp \lp -\frac{2\Delta^{2}}{\sigma^{2}} \E_{i, \pi} \lb n_{a_{i,k},k} \rb \rp.
\end{split}
\end{equation}
Rearranging the terms in \eqref{eq:intv_reg_low_inst_dep}, we obtain
\begin{equation}\label{eq:intv_reg_low_inst_dep_2}
\begin{split}
    \E_{i, \pi} \lb n_{a_{i,k},k} \rb \geq \frac{\sigma^{2}}{2\Delta^{2}} \log \lp \frac{\Delta}{4}\frac{M}{R_{i,k} + R_{j,k}} \rp = \frac{\sigma^{2}}{2\Delta^{2}} \log \lp \frac{\Delta}{4}\frac{T/\lp N_{T} + 1\rp }{R_{i,k} + R_{j,k}} \rp.
\end{split}
\end{equation}
Recall that $R_{T} \leq cT^{p}$ for some $c>0$ and $p\in\lb 0,1 \rp$. Since $R_{i,k}, R_{j,k} \leq R_{T}$, we derive
\begin{equation}\label{eq:intv_reg_low_inst_dep_3}
\begin{split}
    \E_{i, \pi} \lb n_{a_{i,k},k} \rb  &\geq \frac{\sigma^{2}}{2\Delta^{2}} \lb \log  \frac{\Delta}{4} + \log \frac{T}{N_{T} + 1} - \log \lp 2KT^{p} \rp \rb \\
    &= \frac{\sigma^{2}}{2\Delta^{2}} \lb \log  \frac{\Delta}{8c} + \lp 1 - p \rp \log T - \log \lp N_{T} + 1 \rp \rb.
\end{split}
\end{equation}
Thus, the regret on bandit instance $B_{i}$ with the all-zero index vector $i$ is lower bounded as follows:
\begin{align}
\begin{aligned}
    R_{T} &= \sum_{k=1}^{N_{T} + 1} R_{i,k} \\
    &= \sum_{k=1}^{N_{T} + 1} \sum_{a=2}^{A} \Delta \E_{i, \pi} \lb n_{a,k} \rb \\
    &\overset{(a)}{\geq} \sum_{k=1}^{N_{T} + 1} \sum_{a=2}^{A} \Delta \E_{i, \pi} \lb n_{a_{i,k},k} \rb \\
    &\overset{(b)}{\geq} \frac{\sigma^{2}}{2\Delta} \lp N_{T}+1 \rp \lp A - 1 \rp \lb \log  \frac{\Delta}{8c} + \lp 1 - p \rp \log T - \log \lp N_{T} + 1 \rp \rb,
\end{aligned}
\end{align}
where step $(a)$ results from the definition of $a_{i,k}$ in \eqref{eq:ajk}, and step $(b)$ follows from \eqref{eq:intv_reg_low_inst_dep_3}. Note that when deriving the lower bound on $\E_{i, \pi} \lb n_{a_{i,k},k} \rb$ in \eqref{eq:intv_reg_low_inst_dep_2}, the index vector $j$ corresponds to the one-hot vector with the $k^{\mrm{th}}$ bit being $1$. This concludes the proof of Theorem \ref{thm:reg_low_inst_dep}.

To prove Theorem \ref{thm:reg_low_minimax}, we apply the definition of $a_{i,k}$ into \eqref{eq:kl_upp} and obtain
\begin{align}\label{eq:kl_upp_minimax}
    &\mathrm{kl} \lp \Pr_{i',\pi} || \Pr_{j',\pi} \rp=\frac{2\Delta^{2}}{\sigma^{2}} \E_{i, \pi} \lb n_{a_{i,k},k} \rb\overset{(a)}{\leq}\frac{2\Delta^{2}}{\sigma^{2}} \frac{\nu_{k} - \nu_{k-1}}{A-1} = \frac{2\Delta^{2}}{\sigma^{2}} \frac{M}{A-1}.
\end{align}
In step $(a)$, we apply the fact that $\E_{i, \pi} \lb n_{a_{i,k},k} \rb \leq \frac{\nu_{k} - \nu_{k-1}}{A - 1}$, as $a_{i,k} = \argmin_{a \neq 1}\E_{i,\pi} \lb n_{a,k} \rb$. By plugging $\Delta = \sigma\sqrt{\lp A-1 \rp/\lp 4 M \rp}$ into \eqref{eq:kl_upp_minimax}, we have
\begin{align}
    R_{i,k} + R_{j,k} \geq \frac{\sigma}{8e^{1/2}}\sqrt{\lp A-1 \rp M}.
\end{align}
Since the mapping $i \to j$ is bijective and that there are $2^{N_{T}}$ pairs of bandit instances $\lp B_{i}, B_{j} \rp$, we can derive
\begin{align}
    \sum_{i \in \lbp 0,1 \rbp^{N_{T} + 1}} R_{i,k} \geq 2^{N_{T}} \frac{\sigma}{8e^{1/2}}\sqrt{\lp A-1 \rp M}.
\end{align}
Therefore,
\begin{align}
    2^{-\lp N_{T} + 1 \rp} \sum_{i \in \lbp 0,1 \rbp^{N_{T} + 1}} \sum_{k \in \lb N_{T} + 1 \rb} R_{i,k} \geq \frac{\sigma}{16e^{1/2}} \sum_{k \in \lb N_{T} + 1 \rb} \sqrt{\lp A-1 \rp M}
\end{align}
Consequently, there exists a bandit instance $B_{i}$ such that 
\begin{align}
    \sum_{k \in \lb N_{T} + 1 \rb} R_{i,k} \geq \sum_{k \in \lb N_{T} + 1 \rb} \frac{\sigma}{16e^{1/2}}\sqrt{\lp A-1 \rp M} \geq \frac{\sigma}{27}\sqrt{\lp A-1 \rp T \lp N_{T} + 1 \rp}
\end{align}
This concludes the proof of Theorem \ref{thm:reg_low_minimax}.

\section{Proof of Theorem \ref{thm:det_reg}}
\label{sec:thm1}

Consider a PS-MAB environment satisfying Condition \eqref{eq:cond1}, and recall that $\nu_{0} \coloneqq 1$ and $\nu_{N_{T} + 1} \coloneqq T + 1$. We define the following events:
\begin{align}
\mathcal{G}_{k} &\coloneqq\left\{ \forall\,l\in [k-1],\;\tau_{l}\in\lbp \nu_{l}, \dots, \nu_{l} + \ell_{l} - 1\rbp\right\} \cap \lbp \tau_{k} > \nu_{k} \rbp,\, k \in \lb N_{T} \rb. \label{eq:good_event_k}
\end{align}
The event $\mathcal{G}_{k}$ represents the ``good event" up to the $k^{\mathrm{th}}$ detection point $\mcal{G}_{k}$ in which the first $k$ changes are detected within the latency.
For notational convenience, we define $\mcal{G}_{0}$ to be the universal space. The event $\mcal{G}_{3}$ is illustrated in Figure \ref{fig:E_event}.

\begin{figure}[h]
    \centering
    \begin{tikzpicture}[scale=0.5]

        \draw (-14,1)-- (6,1); 
        \node at (-15,1) {$\mcal{G}_{3}$:};
        
        \draw (-14,1.5) -- (-14,0.5);
        \node at (-14,2) {$\nu_{0}$};
        
        \draw (-12,0.5) -- (-12,1.5);
        \node at (-12,2) {$\nu_{1}$};

        \draw (-10.5,1.5) -- (-10.5,1);
        \node at (-10.5,2) {$\tau_{1}$};
        
        \draw[color=blue] (-10,0.5) -- (-10,1);
        \node at (-11,0.2) {{\color{blue} $d_{1}$}};
        
        \draw (-8,0.5) -- (-8,1.5);
        \node at (-8,2) {$\nu_{2}$};

        \draw (-6.8,1.5) -- (-6.8,1);
        \node at (-6.8,2) {$\tau_{2}$};
        
        \draw[color=blue] (-5.2,0.5) -- (-5.2,1);
        \node at (-6.6,0.2) {{\color{blue} $d_{2}$}};
        
        \draw (-2,0.5) -- (-2,1.5);
        \node at (-2,2) {$\nu_{3}$};

        \draw (0.5,1.5) -- (0.5,1);
        \node at (0.5,2) {$\tau_{3}$};
        
        \draw[color=blue] (2,0.5) -- (2,1);
        \node at (0,0.2) {{\color{blue} $d_{3}$}};

        \node at (8,1) {$\cdots$};

        \draw (-14,-4)-- (6,-4); 
        \node at (-16.5,-3.5) {$\mcal{G}^{c}_{3}$:};
        \node at (-16.5,-4.5) {(false alarm)};
        
        \draw (-14,-3.5) -- (-14,-4.5);
        \node at (-14,-3) {$\nu_{0}$};
        
        \draw (-12,-4.5) -- (-12,-3.5);
        \node at (-12,-3) {$\nu_{1}$};

        \draw (-10.5,-3.5) -- (-10.5,-4);
        \node at (-10.5,-3) {$\tau_{1}$};
        
        \draw[color=blue] (-10,-4.5) -- (-10,-4);
        \node at (-11,-4.8) {{\color{blue} $d_{1}$}};
        
        \draw (-8,-3.5) -- (-8,-4.5);
        \node at (-8,-3) {$\nu_{2}$};

        \draw (-6.8,-3.5) -- (-6.8,-4);
        \node at (-6.8,-3) {$\tau_{2}$};
        
        \draw[color=blue] (-5.2,-4.5) -- (-5.2,-4);
        \node at (-6.6,-4.8) {{\color{blue} $d_{2}$}};
        
        \draw (-2,-4.5) -- (-2,-3.5);
        \node at (-2,-3) {$\nu_{3}$};

        \draw[color=red] (-3.5,-3.5) -- (-3.5,-4);
        \node at (-3.5,-3) {{\color{red} $\tau_{3}$}};
        
        \draw[color=blue] (2,-4.5) -- (2,-4);
        \node at (0,-4.8) {{\color{blue} $d_{3}$}};

        \node at (8,-4) {$\cdots$};

        \draw (-14,-9)-- (6,-9); 
        \node at (-16.5,-8.5) {$\mcal{G}^{c}_{3}$:};
        \node at (-16.5,-9.5) {(late detection)};
        
        \draw (-14,-8.5) -- (-14,-9.5);
        \node at (-14,-8) {$\nu_{0}$};
        
        \draw (-12,-9.5) -- (-12,-8.5);
        \node at (-12,-8) {$\nu_{1}$};

        \draw (-10.5,-8.5) -- (-10.5,-9);
        \node at (-10.5,-8) {$\tau_{1}$};
        
        \draw[color=blue] (-10,-9.5) -- (-10,-9);
        \node at (-11,-9.8) {{\color{blue} $d_{1}$}};
        
        \draw (-8,-8.5) -- (-8,-9.5);
        \node at (-8,-8) {$\nu_{2}$};

        \draw[color=red] (-4,-8.5) -- (-4,-9);
        \node at (-4,-8) {{\color{red} $\tau_{2}$}};
        
        \draw[color=blue] (-5.2,-9.5) -- (-5.2,-9);
        \node at (-6.6,-9.8) {{\color{blue} $d_{2}$}};
        
        \draw (-2,-9.5) -- (-2,-8.5);
        \node at (-2,-8) {$\nu_{3}$};

        \draw (0,-8.5) -- (0,-9);
        \node at (0,-8) {$\tau_{3}$};
        
        \draw[color=blue] (2,-9.5) -- (2,-9);
        \node at (0,-9.8) {{\color{blue} $d_{3}$}};

        \node at (8,-9) {$\cdots$};
        
    \end{tikzpicture}
    \caption{Illustration of the event $\mcal{G}$}
    \label{fig:E_event}
\end{figure}
In Figure~\ref{fig:E_event}, the second (false alarm) instance belongs to $\mcal{G}_{2}$, whereas the third (late detection) instance belongs to $\mcal{G}_{1}$. Then, we have the following:
\begin{align}
\begin{aligned}
    R_{T}  & =\E \lb \sum_{k = 1}^{N_{T} + 1} \sum_{t=\nu_{k-1}}^{\nu_{k}-1} \Delta_{a_t,k} \rb \\ 
    & = \sum_{k = 1}^{N_{T} + 1} \E \lb \sum_{t=\nu_{k-1}}^{\nu_{k}-1} \Delta_{a_t,k} \rb \\ 
    &=\sum_{k = 1}^{N_{T} + 1} \lp \Pr \lp \mcal{G}^{c}_{k} \rp \E \lb \sum_{t = \nu_{k-1}}^{\nu_{k} - 1} \Delta_{a_t,k}  \Bigg| \mcal{G}^{c}_{k}\rb 
    + \E \lb \Indc{\mcal{G}_{k}} \sum_{t = \nu_{k-1}}^{\nu_{k} - 1} \Delta_{a_t,k} \rb \rp\\ 
    &\overset{(a)}{\leq} \sum_{k = 1}^{N_{T} + 1} \lp
    C \lp \nu_{k} - \nu_{k-1} \rp \Pr \lp \mcal{G}^{c}_{k} \rp + \E \lb \Indc{\mcal{G}_{k}} \sum_{t = \nu_{k-1}}^{\nu_{k} - 1} \Delta_{a_t, k} \rb \rp\label{eq:reg_decomp}
\end{aligned}
\end{align}
where step $(a)$ results from the fact that $\Delta_{a, k} \leq  C $ for any $k \in \mbb{N}$ and $a \in \lb A \rb$. For convenience in the proof of the upper bound on the probability of bad event $\Pr\lp\mcal{G}^{c}_{k} \rp$, define
\begin{align}
\mathcal{E}_{k} &\coloneqq\left\{ \forall\,l\in [k-1],\;\tau_{l}\in\lbp \nu_{l}, \dots, \nu_{l} + \ell_{l} - 1\rbp\right\},\, k \in \lb N_{T} \rb. \label{eq:good_event_E_k}
\end{align}
$\Pr\lp\mcal{G}^{c}_{k} \rp$ is upper bounded by the following modified union bound, which decomposes the bad event into false alarm events and late detection events:
\begin{align}
    \Pr \lp \mcal{G}^{c}_{k} \rp &= \Pr \lp \lbp \exists\, l \in [k-1],\; \tau_{l} \notin \lbp \nu_{l}, \dots, \nu_{l} + d_{l} - 1 \rbp \rbp \cup \lbp \tau_{k} \leq \nu_{k} \rbp  \rp \nonumber\\ 
    &=\sum_{l = 1}^{k-1} \Pr \lp \tau_{l} \notin \lbp \nu_{s}, \dots, \nu_{l} + d_{l} - 1 \rbp, \mcal{E}_{l-1} \rp  + \Pr \lp \tau_{k} \leq \nu_{k}, \mcal{E}_{k-1} \rp\nonumber\\ 
    &=\sum_{l = 1}^{k-1} \Pr \lp \mcal{E}_{l-1} \rp \Pr \lp \tau_{l} \notin \lbp \nu_{l}, \dots, \nu_{l} + d_{l} - 1 \rbp \big| \mcal{E}_{l-1} \rp + \Pr \lp \mcal{E}_{k-1} \rp \Pr \lp \tau_{k} \leq \nu_{k} \big| \mcal{E}_{k-1} \rp\nonumber\\ 
    &\overset{(a)}{\leq} \sum_{l = 1}^{k-1} \Pr \lp \tau_{l} \notin \lbp \nu_{l}, \dots, \nu_{l}+d_{l} - 1 \rbp \big| \mcal{E}_{l-1} \rp + \Pr \lp \tau_{k} \leq \nu_{k} \big| \mcal{E}_{k-1} \rp\nonumber\\
    &= \sum_{l = 1}^{k} \underbrace{\Pr \lp \tau_{l} < \nu_{l} \big| \mcal{E}_{l-1} \rp}_{\Phi_{1}} + \sum_{l = 1}^{k-1} \underbrace{\Pr \lp \tau_{l} \geq \nu_{l} + d_{l} \big| \mcal{E}_{l-1} \rp}_{\Phi_{2}}\label{eq:bad_union_thm1}
\end{align}
where $(a)$ is due to the fact that $\Pr \lbp \mcal{E}_{k-1} \rbp \leq 1$. We then separately bound $\Phi_{1}$ and $\Phi_{2}$.

$\bullet$ \emph{Upper-Bounding $\Phi_1$:} Recall that $H_{a, \mcal{D}}$ is the change detector history list associated with arm $a$, and that $\tau_{0}$ is defined to be $0$. 
For any $a\in[A]$ and $t>\tau_{l-1}$, we  define $n_{a}\lp t\rp$ to be the number of time-steps between $\tau_{l-1}+1$ and $t$ at which $\lp t - \tau_{l-1} -1 \mod \lce A/\alpha_{l} \rce \rp + 1$ equals $a$, which is the number of samples obtained due to force exploration and added in the history $H_{a,\mathcal{D}}$ if there are no restarts after $\tau_{l-1}$, i.e.,
\begin{align}
    n_{a}\lp t\rp\coloneqq\sum_{s=\tau_{ l-1}+1}^{t}\Indc{\lp t - \tau_{l-1} -1 \mod \lce \frac{A}{\alpha_{l}} \rce \rp + 1 = a}.\label{eq:pull}
\end{align}
We also use $\tau_{a,l}$ to denote the stopping time of the change detector associated with arm $a \in [A]$ after the $(l-1)^{\mrm{th}}$ detection point $\tau_{l-1}$. The stopping time $\tau_{a,l}$ operates independent from other stopping times, and does not stop if other stopping times get triggered earlier.
Then, for all $l\in [N_{T}+1]$, we have
\begin{align}
\begin{aligned}
    \Pr \lp \tau_{l} < \nu_{l} \big| \mcal{E}_{l-1} \rp &= \Pr \lp \exists\, a \in \lb  A \rb, \tau_{a,l} \in \lb n_{a} \lp \nu_{l} - 1 \rp \rb \big| \mcal{E}_{l-1} \rp \\
    &\overset{(a)}{\leq}\sum_{a = 1}^{A} \Pr \lp \tau_{a,l} \in \lb n_{a} \lp \nu_{l} - 1 \rp \rb \big| \mcal{E}_{l-1} \rp \\
    &\overset{(b)}{\leq}\sum_{a=1}^{A}\Pr_{\infty}\lp\tau_{a,l} \leq T\rp\\
    &\overset{(c)}{\leq}\sum_{a=1}^{A}\delta_{\mrm{F}}\\
    &=A\delta_{\mrm{F}}\label{eq:false_alarm}
\end{aligned}
\end{align}
where step $(a)$ results from a union bound. Due to the fact that the rewards after $\tau_{l-1}$ are independent across time-steps and arms given the past event $\mathcal{E}_{l-1}$, we can change the measure to $\Pr_{\infty}$ in step $(b)$, as there are no changes between $\tau_{l-1}$ and $\nu_{l}$. In addition, because $\lb n_{a} \lp \nu_{l} - 1 \rp \rb \subseteq \lb T \rb$, the event $\lbp \tau_{a,l} \in \lb n_{a} \lp \nu_{l} - 1 \rp \rb \rbp \subseteq \lbp \tau_{a,l} \leq T \rbp$. Since the samples are i.i.d. sub-Gaussian, we can apply the false alarm constraint \eqref{eq:PFbound} in step $(c)$.

$\bullet$ \emph{Upper Bounding $\Phi_2$:} Recall that $n_{a} \lp t \rp$ is the number of time-steps between $\tau_{l-1}+1$ and $t$ at which $\lp t - \tau_{l-1} -1 \mod \lce A/\alpha_{l} \rce \rp + 1=a$ (see \eqref{eq:pull}). Then, given $\mathcal{E}_{l-1}$, for any $a \in [A]$ and $i, j > \tau_{l-1}$ where $i < j$, we observe that due to forced exploration,
\begin{align}
\begin{aligned}
    n_{a} \lp j \rp - n_{a} \lp i \rp &= \sum_{t=i+1}^{j}\mathds{1}\lbp\lp t-\tau_{k}-1\mod\lce\frac{A}{\alpha_l}\rce\rp+1=a\rbp\geq\lfl\frac{j-i}{\lce A/\alpha_{l}\rce}\rfl\label{eq:explore}.
\end{aligned}
\end{align}
We observe that $\nu_{l} - \nu_{l - 1} \geq d_{l - 1} + m_{l}$ for any $l \in [N_{T}]$ due to Condition \eqref{eq:cond1}. Then, given $\mcal{E}_{l-1}$, due to the fact that $\tau_{l-1}\leq\nu_{l-1}+d_{l-1} - 1$,
\begin{equation}
    \nu_{l} - 1 \geq \nu_{l-1} + d_{l-1} + m_{l} - 1 \geq \tau_{l-1} + m_{l} > \tau_{l-1}.
\end{equation}
Then, from \eqref{eq:explore} it follows that:
\begin{align}
    n_{a} \lp \nu_{l} + d_{l} - 1 \rp - n_{a} \lp \nu_{l} - 1 \rp \geq \lfl \frac{d_{l}}{\lce A / \alpha_{l} \rce} \rfl = d \lp \underline{\Delta}_{\mrm{c}} \rp \geq d \lp \Delta_{\mrm{c},l} \rp\label{eq:enough_post_change_1}
\end{align}
where  $\Delta_{\mrm{c},l}$ is the change-gap at change-point $\nu_{l}$, and $\underline{\Delta}_{c}$ is the minimum change-gap over all change-points.
Furthermore, we observe that since $\nu_{l} - \nu_{l - 1} \geq d_{l - 1} + m_{l}$  due to Condition \eqref{eq:cond1}, 
\begin{align}
    \nu_{l} - 1 -\tau_{l-1}& \geq \nu_{l}-\nu_{l-1}-d_{l-1} \geq m_{l}\label{eq:enough_pre_change}
\end{align}
due to the fact that $\tau_{l-1}\leq\nu_{l-1}+d_{l-1} - 1$ given $\mcal{E}_{l-1}$. Hence, we can show that given $\mcal{E}_{l-1}$, for any $a \in [A]$,
\begin{align}\nonumber
    n_{a}\lp\nu_{l} - 1\rp &= n_{a}\lp\nu_{l} - 1\rp-n_{a}\lp\tau_{l-1}\rp\\
    &\overset{(a)}{\geq}\lfl\frac{\nu_{l} - 1 - \tau_{l-1}}{\lce A/\alpha_{l}\rce}\rfl\overset{(b)}{\geq}\lfl\frac{m_{l}}{\lce A/\alpha_{l}\rce}\rfl =m\lp\underline{\Delta}_{\mrm{c}}\rp\geq m\lp\Delta_{\mrm{c},l}\rp\label{eq:enough_pre_change_2}
\end{align}
where step $(a)$ results from \eqref{eq:explore} and step $(b)$ results from \eqref{eq:enough_pre_change}. Furthermore, without loss of generality, we can assume that $\nu_{l} \leq T - d_{l}$. Otherwise, there is no need to detect the change because the horizon will end soon after the change occurs. In the case where $\nu_{l} > T - d_{l}$, the regret for not detecting the change $\nu_{l}$ is at most $C d_{l}$, which is also incurred when the change is detected within the latency later in our analysis. Therefore, we have
\begin{align}
    n_{a} \lp \nu_{l} - 1 \rp &\leq \nu_{l} - 1 < T  - d_{l} < T - d_{l} \leq T - \lce \frac{A}{\alpha_{l}} \rce d \lp \Delta_{\mrm{c}, l} \rp \leq T - d \lp \Delta_{\mrm{c}, l} \rp.\label{eq:pre_window_dab}
\end{align}
We define $a_{\mrm{c}, l}$ to be the arm that changes the most at the $l^{th}$ change-point, i.e., $a_{\mrm{c}, l} \coloneqq \argmax_{a = 1, \dots, A}\; \lba \mu_{a, l + 1}-\mu_{a, l} \rba$ for each $l \in [N_{T}]$. By \eqref{eq:enough_pre_change_2} and \eqref{eq:pre_window_dab}, $n_{a_{\mrm{c},l}}\lp \nu_{l} - 1 \rp\in\lbp m\lp\Delta_{\mrm{c},l}\rp+1,\dots,T-d\lp\Delta_{\mrm{c},l}\rp\rbp$ given $\mcal{E}_{l-1}$. Then, we have:
\begin{align}\nonumber
    &\Pr\lp\tau_{l} > \nu_{l} + d_{l} - 1 | \mcal{G}_{l-1} \rp\\\nonumber
    &=\Pr\lp\tau_{l} > \nu_{l} + d_{l} - 1, n_{a_{\mrm{c},l}}\lp\nu_{l} - 1\rp \in \lbp m\lp\Delta_{\mrm{c},l}\rp+1,\dots,T-d\lp\Delta_{\mrm{c},l}\rp\rbp | \mcal{E}_{l-1} \rp\\\nonumber
    &\leq\Pr\lp\forall\,a\in[A],\;\tau_{a,l} > n_{a}\lp\nu_{l} + d_{l} - 1 \rp, n_{a_{\mrm{c},l}}\lp\nu_{l} - 1\rp \in \lbp m\lp\Delta_{\mrm{c},l}\rp+1,\dots,T-d\lp\Delta_{\mrm{c},l}\rp\rbp |\mcal{E}_{l-1} \rp\\\nonumber
    &\overset{(a)}{\leq}\Pr\lp\tau_{a_{\mrm{c},l},l} > n_{a_{\mrm{c},l}}\lp\nu_{l} + d_{l} - 1 \rp, n_{a_{\mrm{c},l}}\lp\nu_{l} - 1\rp \in \lbp m\lp\Delta_{\mrm{c},l}\rp+1,\dots,T-d\lp\Delta_{\mrm{c},l}\rp\rbp |\mcal{E}_{l-1} \rp\\\nonumber
    &\overset{(b)}{\leq} \Pr\lp\tau_{a_{\mrm{c},l},l} > n_{a_{\mrm{c},l}}\lp\nu_{l} - 1 \rp+d\lp\Delta_{\mrm{c},l}\rp, n_{a_{\mrm{c},l}}\lp\nu_{l} - 1\rp \in \lbp m\lp\Delta_{\mrm{c},l}\rp+1,\dots,T-d\lp\Delta_{\mrm{c},l}\rp\rbp |\mcal{E}_{l-1}\rp \\\nonumber
    &=\Pr\lp n_{a_{\mrm{c},l}}\lp\nu_{l} - 1\rp\in\lbp m\lp\Delta_{\mrm{c},l}\rp+1,\dots,T-d\lp\Delta_{\mrm{c},l}\rp\rbp |\mcal{E}_{l-1}\rp\\\nonumber
    &\quad\cdot\Pr\lp\tau_{a_{\mrm{c},l},l}>n_{a_{\mrm{c},l}}\lp\nu_{l} - 1\rp+d\lp\Delta_{\mrm{c},l}\rp | n_{a_{\mrm{c},l}}\lp\nu_{l} - 1\rp \in\lbp m\lp\Delta_{\mrm{c},l}\rp+1,\dots,T-d\lp\Delta_{\mrm{c},l}\rp\rbp,\mcal{E}_{l-1}\rp\\\nonumber
    &\leq \Pr\lp\tau_{a_{\mrm{c},l},l}>n_{a_{\mrm{c},l}}\lp\nu_{l} - 1\rp+d\lp\Delta_{\mrm{c},l}\rp | n_{a_{\mrm{c},l}}\lp\nu_{l} - 1\rp\in \lbp m\lp\Delta_{\mrm{c},l}\rp+1,\dots,T-d\lp\Delta_{\mrm{c},l}\rp\rbp,\mcal{E}_{l-1}\rp\\\nonumber
    &\overset{(c)}{=}\Pr_{\nu}\lp\tau\geq\nu+d\lp\Delta_{\mrm{c},l}\rp\rp\quad\mrm{for \:some\:}\nu\in\lbp m\lp\Delta_{\mrm{c},l}\rp+1,\dots,T-d\lp\Delta_{\mrm{c},l}\rp\rbp\\
    &\overset{(d)}{\leq}\delta_{\mrm{D}} \label{eq:late_detect}
\end{align}
where step $(a)$ comes from the fact that $\lbp a_{\mrm{c},k}\rbp\subseteq[A]$, and step $(b)$ stems from \eqref{eq:enough_post_change_1}. Because $n_{a_{\mrm{c}},l} \lp \nu_{l} - 1 \rp$ is a measurable function of $\tau_{l-1}$, and that the rewards after $\tau_{l-1}$ are independent across time-steps and arms given the $\tau_{l-1}$-mearuable event $\mcal{E}_{l-1}$ and $\lbp n_{a_{\mrm{c},l}}\lp\nu_{l} - 1\rp\in \lbp m\lp\Delta_{\mrm{c},l}\rp+1,\dots,T-d\lp\Delta_{\mrm{c},l}\rp\rbp \rbp$, we can change the measure to $\Pr_{\nu}$ for some $\nu\in\lbp m\lp\Delta_{\mrm{c},l}\rp+1,\dots,T-d\lp\Delta_{\mrm{c},l}\rp\rbp$ in step $(c)$, as there are no changes between $\tau_{l-1}$ and $\nu_{l}$. In addition, because $\lb n_{a} \lp \nu_{k} - 1 \rp \rb \subseteq \lb T \rb$, the event $\lbp \tau_{a,l} \in \lb n_{a} \lp \nu_{l} - 1 \rp \rb \rbp \subseteq \lbp \tau_{a,l} \leq T \rbp$.
Step $(d)$ is due to definition of the latency $d$. This completes bounding $\Phi_1$ and $\Phi_2$.

Plugging \eqref{eq:false_alarm} and \eqref{eq:late_detect} into \eqref{eq:bad_union_thm1}, we obtain
\begin{equation}
\Pr\lbp\mcal{G}^{c}_{k}\rbp\leq Ak\delta_{\mrm{F}}+ \lp k-1 \rp\delta_{\mrm{D}}\label{eq:bad_event}.
\end{equation}
This bounds the first term in \eqref{eq:reg_decomp}. To bound the second term, recall that $\bar{\alpha} \coloneqq \max_{k = 1, \dots, N_{T} + 1} \alpha_{k}$ and that $H_{\mcal{B}}$ is the stationary bandit history list and that $\nu_{N_{T} + 1} \coloneqq T + 1$. For any $k\in\lb N_{T} + 1\rb$, if $\lp t-\tau_{k-1}-1\mod\lce A/\alpha_{k}\rce\rp+1\notin[A]$, then $a_t=\mcal{B}\lp H_{\mcal{B}}\rp$, where $\mcal{B}\lp H_{\mcal{B}}\rp$ denotes the action determined by the stationary bandit algorithm $\mcal{B}$ with history $H_{\mcal{B}}$. Thus, the second term in \eqref{eq:reg_decomp} can then be decomposed as follows:
\begin{align}
    &\E\lb \Indc{\mcal{G}_{k}}\sum_{t=\nu_{k-1}}^{\nu_{k}-1}\Delta_{a_t, k}\rb \nonumber \\
    &\overset{(a)}{\leq} Cd_{k-1}+\lce\frac{\nu_{k}-\nu_{k-1}}{\lce A/\alpha_{k}\rce}\rce AC + \E\lb\Indc{\mcal{G}_{k}}\sum_{t=\tau_{k-1} + 1}^{\nu_{k}-1}\Delta_{a_t, k}\Indc{\lp t-\tau_{k-1}-1 \rp\mod\lce \frac{A}{\alpha_{k}}\rce \geq A}\rb \nonumber \\
    &\overset{(b)}{\leq} C d_{k-1} + C \lb \alpha_{k} \lp \nu_{k}-\nu_{k-1}\rp + A \rb + R_{\mcal{B}}\lp \nu_{k}-\nu_{k-1}\rp \nonumber \\
    &\leq C d_{k-1} + C \lb \bar{\alpha} \lp \nu_{k}-\nu_{k-1}\rp + A \rb +R_{\mcal{B}} \lp \nu_{k}-\nu_{k-1} \rp
    \label{eq:bound_inter}
\end{align}
where in step $(a)$, the first term bounds the regret due to the delay of the change detector, and the second term bounds the regret incurred due to forced exploration. In step $(b)$, as the reward samples in $\mcal{H}_{\mcal{B}}$ are independent of those in $\cup_{a\in[A]}\mcal{H}_{a,\mcal{D}}$, and that $\mcal{G}_{k}$ only depends on samples in $\cup_{a\in[A]}\mcal{H}_{a,\mcal{D}}$,  the regret bound of the stationary bandit    
\footnote{It is worth noting what would happen if we allow the stationary bandit algorithm to have access to samples obtained from forced exploration as is done in \cite{besson2022efficient}. Suppose $\tilde{R}_{\mcal{B}}$ is the regret of the stationary bandit that uses the forced exploration samples, then the second and third terms in the bound in step (a)  will be replaced by $\sum_{k=1}^{N_{T} + 1}\tilde{R}_{\mcal{B}}\lp \nu_{k}-\nu_{k-1}-d_{k-1}\rp$, which is $\geq \sum_{k=1}^{N_{T} + 1}R_{\mcal{B}}\lp \nu_{k}-\nu_{k-1}-d_{k-1}\rp$ since the forced exploration is a suboptimal way to pull the arms for the stationary bandit. This would make it difficult to upper bound the bandit regret in a manner similar to step (b) and the regret bound will no longer be modular.} can be applied to the third term.   
We also apply Property \ref{proper:regret}, utilizing the fact that $R_{\mcal{B}}\lp T \rp$ is increasing with $T$. Then, we can plug \eqref{eq:bound_inter} and \eqref{eq:bad_event} into \eqref{eq:reg_decomp} and obtain:
\begin{align}
&R_{T} \nonumber\\
&\leq \sum_{k = 1}^{N_{T} + 1}
    C \lp \nu_{k} - \nu_{k-1} \rp \lp Ak\delta_{\mrm{F}}+ \lp k-1 \rp\delta_{\mrm{D}} \rp  + \sum_{k = 1}^{N_{T} + 1} \lp C d_{k-1} + C \lb \bar{\alpha} \lp \nu_{k}-\nu_{k-1}\rp + A \rb +R_{\mcal{B}} \lp \nu_{k}-\nu_{k-1} \rp \rp \nonumber\\
    &\leq \sum_{k = 1}^{N_{T} + 1}
    C \lp \nu_{k} - \nu_{k-1} \rp \lp A\lp N_{T}+1 \rp\delta_{\mrm{F}}+ N_{T}\delta_{\mrm{D}} \rp + \sum_{k = 1}^{N_{T} + 1} \lp C d_{k-1} + C \lb \bar{\alpha} \lp \nu_{k}-\nu_{k-1}\rp + A \rb +R_{\mcal{B}} \lp \nu_{k}-\nu_{k-1} \rp \rp \nonumber\\
    &=C T A\lp N_{T} + 1\rp\delta_{\mrm{F}} + C T N_{T} \delta_{\mrm{D}} + C \sum_{k=1}^{N_{T}} d_{k} + C\lb \bar{\alpha} T + \lp N_T + 1 \rp A \rb + \sum_{k = 1}^{N_{T} + 1} R_{\mcal{B}} \lp \nu_{k}-\nu_{k-1} \rp \nonumber\\
    &\overset{(a)}{\leq} C T A\lp N_{T} + 1\rp\delta_{\mrm{F}} + C T N_{T} \delta_{\mrm{D}} + C \sum_{k=1}^{N_{T}} d_{k} + C\lb \bar{\alpha} T + \lp N_T + 1 \rp A \rb + \lp N_{T} + 1 \rp R_{\mcal{B}}\lp \frac{T}{N_{T}+1} \rp.
\end{align}
In step $(a)$, we apply Property \ref{proper:regret}, applying Jensen's inequality to the concave function $R_{\mcal{B}}$. This concludes the proof of Theorem \ref{thm:det_reg}.

\end{document}